\documentclass{article}

\usepackage[utf8]{inputenc} 
\usepackage[T1]{fontenc}    
\usepackage{hyperref}       
\usepackage{url}            
\usepackage{booktabs}       
\usepackage{amsfonts}       
\usepackage{nicefrac}       
\usepackage{microtype}      
\usepackage{lipsum}		
\usepackage[pdftex]{graphicx}
\usepackage{doi}

\usepackage{multirow}
\usepackage[utf8]{inputenc} 
\usepackage[T1]{fontenc}    
\usepackage{hyperref}       
\usepackage{url}            
\usepackage{booktabs}       
\usepackage{amsfonts}       
\usepackage{nicefrac}       
\usepackage{microtype}      
\usepackage{xcolor}         
\usepackage{caption}
\usepackage{subcaption}
\usepackage{mathtools}
\usepackage[letterpaper,top=2cm,bottom=2cm,left=3cm,right=3cm,marginparwidth=1.75cm]{geometry}

\usepackage{derivative}

\usepackage{blindtext}
\usepackage{amsmath}
\usepackage{amssymb}
\usepackage{amsthm}
\usepackage{bm}
\usepackage{thmtools,thm-restate}
\usepackage[shortlabels]{enumitem}

\newcommand{\bz}{\mathbf{z}}

\newcommand{\btheta}{\mathbf{\theta}}

\newcommand{\by}{\mathbf{y}}

\def\bphi{{\boldsymbol{\phi}}}

\usepackage{mathtools}

\def\eandelong{{\sc Embed \& Emulate}\xspace}
\def\eande{{E\&E}\xspace}
\makeatletter
\newsavebox\myboxA
\newsavebox\myboxB
\newlength\mylenA
\def\eandelong{{\sc Embed \& Emulate}\xspace}
\def\eande{{E\&E}\xspace}

\DeclareMathOperator*{\E}{\mathbb{E}}
\DeclareMathOperator*{\argmin}{\mathrm{argmin}}

\usepackage{multirow}
\usepackage{comma}
\usepackage{algorithm, algorithmic}
\usepackage{pifont}
\usepackage{subcaption}
\usepackage{booktabs}
\PassOptionsToPackage{sort}{natbib}
\usepackage{hyperref}       
\definecolor{midnight}  {rgb}{0,0,.5}
\definecolor{forest}  {rgb}{0,.4,0}
\hypersetup{
  colorlinks,
  citecolor=forest,
  linkcolor=black,
  urlcolor=midnight}

\usepackage{lastpage}
\usepackage{algorithm, algorithmic}

\usepackage[capitalize,noabbrev]{cleveref}

\title{Embed and Emulate: Contrastive representations for simulation-based inference}
\usepackage{authblk}

\author[1]{Ruoxi Jiang\thanks{Equal contribution}}
\author[1]{Peter Y. Lu$^*$}
\author[1,2]{Rebecca Willett}
\affil[1]{Department of Computer Science, The University of Chicago, US}
\affil[2]{Department of Statistics, The University of Chicago, US}
\date{}                   
\newtheorem{theorem}{Theorem}
\newtheorem{lemma}[theorem]{Lemma}

\newtheorem{corollary}[theorem]{Corollary}
\newtheorem{definition}[theorem]{Definition}

\usepackage{natbib}

\newenvironment{keywords}
{\bgroup\leftskip 20pt\rightskip 20pt \small\noindent{\bf Keywords:} }%
{\par\egroup\vskip 0.25ex}

\begin{document}
\maketitle

\begin{abstract}
Scientific modeling and engineering applications rely heavily on parameter estimation methods to fit physical models and calibrate numerical simulations using real-world measurements. In the absence of analytic statistical models with tractable likelihoods, modern simulation-based inference (SBI) methods first use a numerical simulator to generate a dataset of parameters and simulated outputs. This dataset is then used to approximate the likelihood and estimate the system parameters given observation data. Several SBI methods employ machine learning emulators to accelerate data generation and parameter estimation. However, applying these approaches to high-dimensional physical systems remains challenging due to the cost and complexity of training high-dimensional emulators. This paper introduces Embed and Emulate (\eande): a new SBI method based on contrastive learning that efficiently handles high-dimensional data and complex, multimodal parameter posteriors. \eande learns a low-dimensional latent embedding of the data (i.e., a summary statistic) and a corresponding fast emulator in the latent space, eliminating the need to run expensive simulations or a high-dimensional emulator during inference. We illustrate the theoretical properties of the learned latent space through a synthetic experiment and demonstrate superior performance over existing methods in a realistic, non-identifiable parameter estimation task using the high-dimensional, chaotic Lorenz 96 system.
\end{abstract}

\begin{keywords}
simulation-based inference, contrastive representations, parameter estimation, multimodal distributions, high-dimensional data
\end{keywords}

\section{Introduction}
Model parameter estimation and inference tasks are ubiquitous in scientific modeling and engineering applications, spanning fields such as climate forecasting \citep{schneider2017earth, adam2018regional, schneider2024opinion}, cosmology \citep{mishra2022neural, boddy2022snowmass2021, Prat_2023}, evolutionary biology \citep{toni2010simulation, st2019bayesian}, and more. 
In scientific modeling, parameter inference is used to fit physical models to real observation data, providing not only point estimates of system parameters but also measures of uncertainty that are critical for science.
However, traditional statistical inference methods often fall short when dealing with complex, high-dimensional models that lack tractable likelihood functions.
Simulation-based inference (SBI) has emerged as a powerful approach to address these challenges \citep{cranmer2020frontier, lueckmann2021benchmarking}. 
SBI uses numerical simulators to generate data for a range of parameter values, enabling inference without explicit likelihood calculations.
For example, approximate Bayesian computation (ABC) \citep{csillery2010approximate}, a traditional SBI method, iteratively compares simulated data with observations to construct parameter estimates and posterior distributions. 
These estimates provide key insights into the behavior of the physical system and allow us to accurately calibrate physical models.

However, traditional SBI methods like ABC can be prohibitive when the underlying numerical simulators are computationally demanding. 
Machine learning provides a mechanism for mitigating this challenge via \textit{learned emulators}, which are data-driven models trained to mimic numerical simulations at a much lower computational cost \citep{li2020fourier, gupta2022multispatiotemporalscalegeneralizedpdemodeling, pentland2023gparareal, takamoto2023learning, bruna2024neural, jiang2024training, schiff2024dyslim, raonic2024convolutional, cachay2024probabilistic}.
In the context of SBI, one might consider the following procedure \citep{lueckmann2019likelihood}: 
first, generate training samples using a numerical simulator for a physical system; then, learn an emulator that maps system parameters to the simulated data using the generated samples; finally, given real observation data, find the best-fit parameters using the (cheap) emulator in place of the original (expensive) numerical simulator.
Past efforts have shown compelling proofs of concept using emulators for parameter estimation \citep{RAISSI2019686, lueckmann2019likelihood, gmd-2021-267, watson2021model}, but a number of open challenges remain. For example, the reported computational savings associated with the learned emulators typically do not account for the computational burden of generating training data using an expensive numerical simulator. Furthermore, when the simulator outputs are high-dimensional, the number of training samples needed for high-fidelity predictions may be quite large. 

This paper describes an alternative approach, based on contrastive learning \citep{hoffer2015deep, oord2018representation, zhang2020self, radford2021learning, zhang2022dino}, that aims to reduce the number of training samples required for SBI by designing an emulator specifically for the parameter estimation task. Our approach, called \textit{Embed and Emulate} (\eande), jointly learns a low-dimensional latent embedding for the data and a fast latent emulator that maps the system parameters directly to the latent space. The embedding homes in on aspects of the simulator outputs that are most salient to the parameter estimation task---the learned embedding is a low-dimensional summary statistic. 
During inference, we first embed the data (i.e.,~compute the learned summary statistic) and then use the latent emulator to estimate the system parameters. The embedding and the latent emulator can be learned with far fewer samples than an emulator operating in the system's original high-dimensional output space. In fact, the computational burden of generating training data and performing inference are both much lower using the \eande approach than with SBI methods operating in the original high-dimensional space, facilitating accurate and efficient parameter estimation even for complex physical systems.

In this work, we extend the scope of the original \eande method \citep{jiang2022embed} to address more general SBI tasks and provide new theoretical support for the approach.
While the original approach was limited by its assumption of a Gaussian likelihood \citep{iglesias2013ensemble} and its use of a supervised regression head for unimodal parameter estimation, our extension significantly broadens its applicability.
We update key elements of the original method and propose a parameterization of the posterior that allows for a wide class of multimodal likelihoods.
Another major contribution of our work is the theoretical justification.
We motivate and significantly clarify the use of contrastive learning in this context using new theoretical results that tie optimizing the contrastive loss directly to estimating the posterior parameter distribution. Altogether, these advancements allow \eande to estimate posterior distributions for system parameters that provably converge to the true posterior given sufficient data.

\subsection{Contributions}
\begin{enumerate}
    \item We propose a new method for simulation-based inference with a focus on high-dimensional data. Inspired by contrastive representation learning, our approach parameterizes the likelihood-to-evidence ratio as a distance in a latent representation space between a data embedding and a parameter embedding. The embeddings are first trained using a symmetric inter-domain contrastive loss, which ensures that we capture the correct likelihood-to-evidence ratio, and an optional intra-domain contrastive loss, which reduces the variance of parameter estimates. The learned likelihood-to-evidence ratio, along with the prior, can then be used to construct and sample from the posterior parameter distribution.
    \item We develop a theoretical framework for our method based on an analysis of contrastive representation learning and show that, under standard assumptions, the estimated posterior derived from our learned embeddings will converge to the true posterior. Our analysis also highlights the distinctive features of our approach: the data and parameter embeddings, which can be interpreted as a sufficient statistic and a fast latent emulator, and the symmetric form of the inter-domain loss, which results in a better-behaved estimate of the posterior.
    \item We use a synthetic task to illustrate our theoretical results in a controlled setting and provide additional insight into the learned data and parameter embeddings. 
    In this setting, we can explicitly predict the form of the optimal embeddings and reconstruct the latent generative process. We also introduce additional redundant parameters to study settings where the parameters are not identifiable. Our method learns to ignore the redundant parameters, leading to low-dimensional latent embeddings.
    \item We then test our method on a realistic simulation-based inference task using the Lorenz 96 system: a high-dimensional, chaotic model for atmospheric dynamics. In our experiment, we examine a parameterization of the dynamical system that introduces parameter non-identifiability and thus a complex posterior distribution. On this challenging task, our method outperforms other recently proposed simulation-based inference methods.
\end{enumerate}
\subsection{Related work}
\textbf{Simulation-based inference (SBI).} Traditional SBI methods like Approximate Bayesian Computation (ABC) have been widely used but suffer from limitations including sample inefficiency, reliance on predefined sufficient statistics, and lack of amortization.
ABC, for instance, typically requires repeated simulator runs and must restart its costly inference process for each new observation.
To address these limitations, modern techniques propose training tractable surrogates or emulators for simulators \citep{cranmer2020frontier, spurio2023bayesian}.
These approaches can be broadly classified based on their estimation targets: neural posterior estimation (NPE) \citep{papamakarios2016fast,lueckmann2017flexible,greenberg2019automatic,rodrigues2021hnpe,ward2022robust}, neural likelihood estimation (NLE) \citep{papamakarios2019sequential}, and neural ratio estimation (NRE) \citep{moustakides2019training, hermans2020likelihood, miller2021truncated, miller2022contrastive, kelly2024misspecification}.
Furthermore, these approaches can be divided into two setups: sequential and amortized.
The sequential setup involves iteratively updating the proposal distribution by generating new simulations, while the amortized setup aims to learn the posterior distribution from a fixed dataset, thereby spreading the cost of training and data simulation across multiple observations.
In this study, we focus on the amortized setup, which is primarily employed for NRE methods and is also applicable for NPE methods.

NPE methods directly parameterize posteriors using conditional density estimators with specialized architectures like mixture density networks, autoregressive models, and normalizing flows \citep{bishop1994mixture, uria2016neural, papamakarios2017masked, kobyzev2020normalizing, papamakarios2021normalizing}.
The initial method (sequential) NPE-A \citep{papamakarios2016fast} uses mixture density networks to fit the proposal distribution by minimizing the negative log probability of the simulated parameters given the corresponding data.
In its sequential setup, when the proposal distribution for training samples is different from the prior, they propose an analytical post-hoc step for correction. 
The subsequent method SNPE-B \citep{lueckmann2017flexible} alleviates the need for post-hoc corrections by introducing an importance-weighted loss, although a high variation of the importance weights during training can lead to inaccurate inference.
SNPE-C \citep{greenberg2019automatic} addresses these issues through a reparameterization that also supports more flexible architectures, such as normalizing flows. 
In this study, we refer to its single-round non-sequential version as NPE-C.

Instead of direct posterior estimation, NRE methods learn the posterior by estimating the ratio between the data-generating distribution and the marginal distribution, referred to as the likelihood-to-evidence ratio \citep{cranmer2015approximating}.
Early works LFIRE \citep{thomas2020likelihoodfreeinferenceratioestimation} and NRE-A \citep{hermans2020likelihood} frame this ratio estimation problem as a classification task, where the goal is to distinguish between samples generated from the joint distribution (parameters and simulated data sampled together) and samples from the marginal distribution (parameters and data sampled independently).
While LFIRE learns a separate classifier per posterior evaluation, the latter method NRE-A trains an amortized classifier and improves computational efficiency.
{\sc NRE-B} \citep{durkan2020contrastive} extends this binary classification framework to a multi-class task, improving stability and accuracy.
However, \citet{miller2022contrastive} identify a normalization problem of {\sc NRE-B}: the estimated ratio includes a data-dependent bias term. 
To tackle this, they introduce {\sc NRE-C} using a similar framework but correct the ratio by introducing additional hyperparameters.

Despite these advancements, scaling {\sc NPE} or {\sc NRE} methods to high-dimensional problems remains challenging \citep{wildberger2024flow, anau2024scalable, akhmetzhanova2024data, gloeckler2024all}.
For NPE methods, normalizing flows are often employed to model the posterior distribution, but such architectures are known to suffer from the curse of dimensionality due to the invertibility constraint \citep{papamakarios2021normalizing}.
While {\sc NPE-C} attempts to address this computation bottleneck by implicitly attaching an embedding network in front of the constrained neural density estimator, learning embeddings without explicit objectives to form a structured latent space can require large quantities of training data.
Our \eande method mitigates these challenges by choosing a loss function that provides explicit guidance for learning a structured embedding space to enable accurate parameter estimation and an embedding network framework designed for data efficiency and computational efficiency.
When compared to neural ratio estimation methods, we show that our ratio estimator is properly normalized without introducing extra hyperparameters compared to NRE-C.

\vspace{5pt}
\noindent\textbf{Contrastive representation learning.} 
Contrastive learning has emerged as a powerful paradigm in unsupervised representation learning \citep{hoffer2015deep, goroshin2015unsupervised, oord2018representation, he2020momentum, grill2020bootstrap, zhou2021ibot, rangnekar2022semantic, zhang2022dino}.
Operating under the push-pull principle, these methods pull ``positive'' data points (formed under various data augmentations) closer together in the latent space while simultaneously pushing ``negative'' data points apart.
The majority of these frameworks, focusing on instance discrimination in vision, have primarily dealt with single-domain data \citep{chen2020simple, he2020momentum}, where positive pairs are typically created through different augmentations of the same image.
Recently, cross-domain methods like CLIP \citep{radford2021learning} and ALIGN \citep{jia2021scaling} have extended this framework across different domains, often pairing images with textual descriptions.
Theoretical understanding of contrastive learning has also advanced. 
\citet{wang2020understanding} provides a theoretical framework for understanding single-domain contrastive learning, relating it to the alignment and uniformity of the latent features on the hypersphere,
and \citet{zimmermann2021contrastive} interpret single-domain contrastive learning as learning to invert a latent generative process.

In this work, we adapt cross-domain contrastive learning and build on prior theoretical analyses to develop a new method for SBI that parameterizes the posterior in terms of learned contrastive representations.

\section{Problem setup}
High-fidelity simulations play a fundamental role in modeling and understanding complex physical processes. Simulations typically model a stochastic generative process $\by \sim G(\bphi)$, taking input parameters $\bphi \in \Phi \subseteq \mathbb{R}^s$ and sampling data $\by \in \mathcal{Y} \subseteq \mathbb{R}^d$ with a likelihood distribution $p(\by \mid \bphi)$. The scientific goal is then to match this model to observed data $\by^o$ by estimating the parameters $\bphi$ associated with the real physical system with uncertainty. Formally, this problem can be solved using Bayesian inference by computing the posterior parameter distribution
\begin{align}
    p(\bphi \mid \by^o) = \frac{p(\by^o \mid \bphi)\,p(\bphi)}{p(\by^o)}
\end{align}
given the model likelihood $p(\by \mid \bphi)$ and a parameter prior $p(\bphi)$. Often, this is accomplished by sampling from the posterior using methods, such as Markov chain Monte Carlo (MCMC) or rejection sampling \citep{eckhardt1987stan}, that use an unnormalized form of the distribution $p(\bphi \mid \by^o) 
 \propto p(\by^o \mid \bphi)\,p(\bphi)$.

However, for complex scientific models, standard Bayesian inference approaches run into computational difficulties. In many cases, the likelihood function $p(\by \mid \bphi)$ cannot be explicitly derived or efficiently computed. Instead, simulation-based inference (SBI), also known as likelihood-free inference, uses the simulator to generate a dataset of parameter--data pairs $\{(\bphi_i, \by_i)\}_{i=1}^N \sim p(\bphi, \by)$, which can then be used to estimate the likelihood \citep{cranmer2020frontier} or a related quantity such as the likelihood-to-evidence ratio \citep{thomas2020likelihoodfreeinferenceratioestimation, hermans2020likelihood, durkan2019neural, miller2022contrastive}
\begin{equation}
r(\bphi, \by) := \frac{p(\by \mid \bphi) } {p(\by)} = \frac{p(\bphi \mid \by)} {p(\bphi)} = \frac{p(\bphi, \by)}{p(\bphi)\,p(\by)}.
\end{equation}
This problem is particularly challenging in scenarios where the data $\by$ is high-dimensional, and therefore the likelihood $p(\by\mid\bphi)$ is a high-dimensional probability distribution. High-dimensional data is common in many scientific domains \citep{schneider2017earth, wang2022physics, ye2024nonlinear}, such as spatiotemporal dynamical systems, and is one of the key difficulties that we aim to address in our new SBI approach.
\section{Proposed approach}
Directly emulating the generative process $\by \sim G(\bphi)$ and approximating the likelihood $p(\by\mid\bphi)$ for high-dimensional data $\by$ is challenging and often has high sample complexity, requiring complex models trained on large datasets \citep{bi2022pangu,lam2022graphcast, kurth2023fourcastnet, li2024generative}. Instead, our proposed \textit{Embed and Emulate} (\eande) method simultaneously learns an encoder $\hat{f}_\theta:\mathcal{Y}\to\mathbb{S}^{n-1}$ for compressing the high-dimensional data $\by\in\mathcal{Y}$ to a summary statistic in a latent space $\mathbb{S}^{n-1}$ (the unit hypersphere in $\mathbb{R}^n$) and a latent emulator $\hat{g}_\theta:\Phi\to\mathbb{S}^{n-1}$ for learning to map the parameters $\bphi\in\Phi$ to the summary statistic.

The encoder $\hat{f}_\theta$ performs dimensionality reduction, removing irrelevant information that is not necessary for parameter inference and thus providing a low-dimensional sufficient statistic (Section~\ref{sec:maintheory}). The latent emulator $\hat{g}_\theta  \approx \hat{f}_\btheta \circ G$ can then focus on emulating the low-dimensional statistic rather than the original high-dimensional data. These two components $\hat{f}_\theta,\hat{g}_\theta$ are jointly trained to reconstruct the likelihood-to-evidence ratio $r(\bphi,\by)$, which then allows us to sample from the parameter posterior $p(\bphi \mid \by) = r(\bphi, \by)\,p(\bphi)$. Since we learn the likelihood-to-evidence ratio rather than the posterior directly, the prior used for data generation does not necessarily need to match the prior for inference, allowing for greater flexibility (Appendix~\ref{sec:altpriors}).

\subsection{Parameterizing the likelihood-to-evidence ratio as a similarity measure}
The emulator $\hat{g}_\theta$ ideally learns a map from the parameters $\bphi$ to the summary statistic $\hat{g}_\theta(\bphi) \approx \hat{f}_\btheta(G(\bphi))$ with mismatch due to the stochastic nature of the generative process $G$.
When the marginals $p(\bphi)$ and $p(\by)$ are fixed, the likelihood-to-evidence ratio $r(\bphi, \by)$ scales proportionally to the joint density $p(\bphi, \by)$.
Therefore, given data $\by$, we expect $r(\bphi, \by)$ to be large when $\hat{g}_\theta(\bphi)$ is ``close'' to $\hat{f}_\btheta(\by)$ and small when $\hat{g}_\theta(\bphi)$ is ``far'' from $\hat{f}_\btheta(\by)$. We can make this intuition precise by parameterizing $r(\bphi, \by)$ in terms of a metric or other similarity measure in the latent embedding space $\mathbb{S}^{n-1}$. In particular, we parameterize our model for the likelihood-to-evidence ratio
\begin{equation}
    \log \hat{r}_\theta(\bphi, \by) = \hat{f}_\theta(\by)\cdot\hat{g}_\theta(\bphi)/\tau - \log C(\by)
\end{equation}
in terms of the cosine similarity $\hat{f}_\theta(\by)\cdot\hat{g}_\theta(\bphi)$ with a scale hyperparameter $\tau$ and normalization factor $C(\by)$. From this, we can derive a model for the posterior distribution
\begin{equation}\label{eqn:posterior}
    \hat{q}_\theta(\bphi \mid \by) := \hat{r}_\theta(\bphi, \by)\,p(\bphi) = C(\by)^{-1}e^{\hat{f}_\theta(\by)\cdot\hat{g}_\theta(\bphi)/\tau}\,p(\bphi).
\end{equation}
By choosing this form for the likelihood-to-evidence ratio, we are effectively picking a particular form of the likelihood distribution $p(\hat{f}_\btheta(\by)\mid\bphi)$ for the summary statistic, which determines the structure of the learned latent space (Section \ref{sec:understanding}). This form can also be equivalently described as parameterizing the posterior (\ref{eqn:posterior}) as an exponential family distribution with sufficient statistic $\hat{f}_\theta(\by)$ and natural parameter $\hat{g}_\theta(\bphi)$.

\subsection{Optimizing the symmetric inter-domain InfoNCE loss}
\label{sec:inter-modal}
\begin{figure}[t!]
    \centering
    \includegraphics[width=1\linewidth]{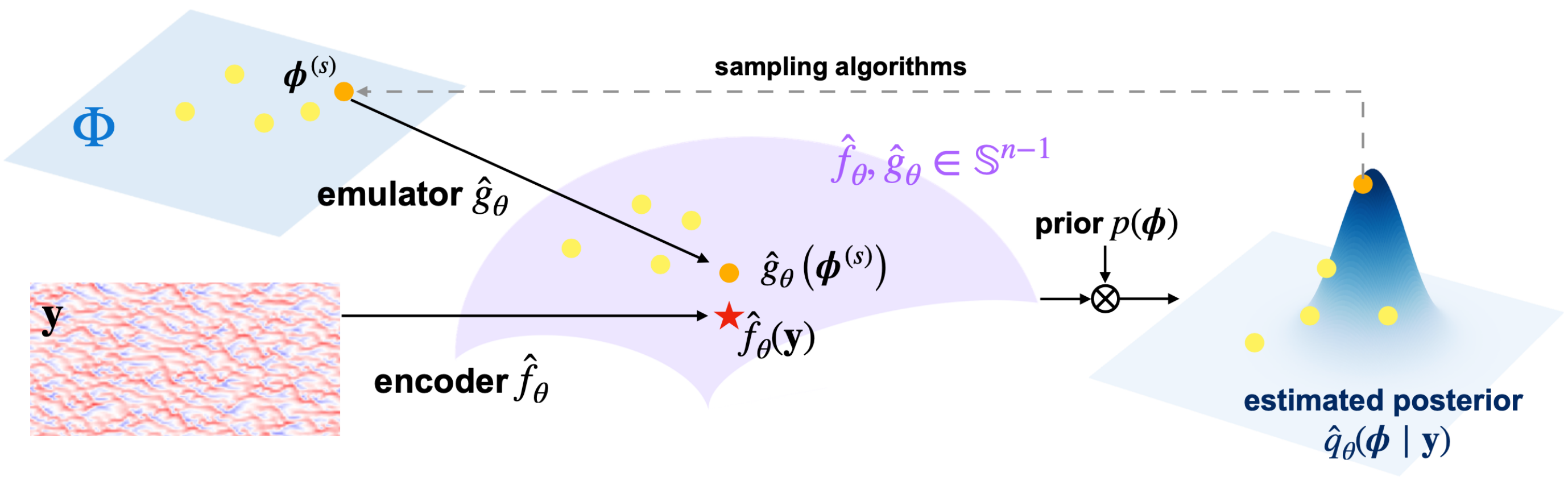}
    \caption{\textbf{Diagram of posterior inference using \eande.} For each observation $\by$, our approach requires a single forward pass through the encoder $f_\theta$, which processes the high-dimensional data to produce a lower-dimensional embedding. This embedding is then combined with the emulator’s outputs $\hat{g}_\theta(\bphi)$ to parameterize the posterior estimator. 
    Posterior samples $\bphi^{(s)} \sim p(\bphi)$ are drawn using a posterior sampling algorithm that repeatedly calls the learned emulator $\hat{g}_\theta$.}
    \label{fig:eande_diagram_sampling}
\end{figure}
We aim to optimize the model $\hat{r}_\theta(\bphi, \by)$ so that it matches the true likelihood-to-evidence ratio $r(\bphi, \by)$, which acts like a similarity measure that is large for matched positive pairs $(\bphi, \by) \sim p(\bphi, \by) = p(\by\mid\bphi)\,p(\bphi)$ sampled from the joint distribution and small for unmatched negative pairs $(\bphi, \by) \sim p(\by)\,p(\bphi)$. Therefore, we aim to learn embedding functions $\hat{f}_\theta,\hat{g}_\theta$ such that positive pairs are embedded nearby in latent space and negative pairs are embedded further apart. The symmetric inter-domain InfoNCE loss, also known as CLIP \citep{radford2021learning} in contrastive representation learning, formalizes this intuition: 

\begin{align}
   L_\mathrm{sym}(\hat{f}_\btheta, \hat{g}_\btheta, M)
   & := L_{\Phi\mathcal{Y}}(\hat{f}_\btheta, \hat{g}_\btheta, M) 
   + L_{\mathcal{Y}\Phi}(\hat{f}_\btheta, \hat{g}_\btheta, M),
   \label{eqn:L_inter}
\end{align}
where
\begin{align}
    L_{\Phi\mathcal{Y}}(\hat{f}_\btheta, \hat{g}_\btheta, M) 
    &:= -\frac{1}{M}\sum_{i=1}^M\log\left(\frac{\hat{r}_\theta(\bphi_i, \by_i)}{\sum_{j=1}^M\hat{r}_\theta(\bphi_j, \by_i)}\right) = -\frac{1}{M}\sum_{i=1}^M\log\left(\frac{e^{\hat{f}_\btheta(\by_i)\cdot\hat{g}_\btheta(\bphi_i)/\tau}}{\sum_{j=1}^M e^{\hat{f}_\btheta(\by_i)\cdot\hat{g}_\btheta(\bphi_j)/\tau}}\right)
    \label{eqn:LPhiY}\\
    L_{\mathcal{Y}\Phi}(\hat{f}_\btheta, \hat{g}_\btheta, M) 
    &:= -\frac{1}{M}\sum_{i=1}^M\log\left(\frac{\hat{r}_\theta(\bphi_i, \by_i)}{\sum_{j=1}^M\hat{r}_\theta(\bphi_i, \by_j)}\right) = -\frac{1}{M}\sum_{i=1}^M\log\left(\frac{e^{ \hat{f}_\btheta(\by_i)\cdot\hat{g}_\btheta(\bphi_i)/\tau}}{\sum_{j=1}^M e^{\hat{f}_\btheta(\by_j)\cdot\hat{g}_\btheta(\bphi_i)/\tau}}\right)
    \label{eqn:LYPhi}
\end{align}
are computed using a batch of training parameter--data pairs $\{(\bphi_i,\by_i)\}_{i=1}^M \overset{\text{i.i.d.}}{\sim} p(\bphi,\by)$ of size $M$. $L_{\Phi\mathcal{Y}}$ trains $\hat{r}_\theta$ to match a given $\by_i$ to the corresponding $\bphi_i$ from a batch of parameters $\{\bphi_j\}_{j=1}^M$, and, similarly, $L_{\mathcal{Y}\Phi}$ trains $\hat{r}_\theta$ to match a given $\bphi_i$ to the corresponding $\by_i$ from a batch of data samples $\{\by_j\}_{j=1}^M$. 
We present the main training algorithm in Section \ref{sec:algorithm-main}.
Together, these loss functions align positive pairs in the latent space while separating negative pairs (Appendix~\ref{sec:asymptotics}). In fact, we prove that, at the global optimum of this symmetric InfoNCE loss, the trained model $\hat{r}_\theta(\bphi,\by)$ exactly matches the true likelihood-to-evidence ratio $r(\bphi,\by)$ (Section~\ref{sec:maintheory}).

\vspace{5pt}
\noindent\textbf{Symmetric form of the InfoNCE loss.} 
Requiring symmetry in forming negative pairs from two distinct domains (e.g.~parameters and data, or text and images) leads to the symmetric InfoNCE loss \citep{radford2021learning}. In contrastive representation learning, this symmetric loss aligns embeddings from different domains while ensuring the embeddings remain as distinctive as possible by balancing both types of negative pairs. In our \eande method, the symmetric form of the loss acts as a regularizer that ensures that the estimated ratio $\hat{r}_\theta(\bphi, \by)$ recovers the true likelihood-to-evidence ratio $r(\bphi, \by)$ up to a normalization constant $C(\by) = C^*$ that does not vary with the data $\by$ (Section~\ref{sec:maintheory}). Without this guarantee, $C(\by)$ can have an arbitrarily pathological dependence on $\by$, which leads to poor empirical performance \citep{ma2018noise} and makes verification tools like importance sampling-based diagnostics unusable \citep{miller2022contrastive}.

\vspace{5pt}
\noindent \textbf{Efficient sampling from the model posterior.}
Our design choice of $\hat{r}_\theta(\bphi, \by)$ offers significant computational advantages for posterior inference, particularly when dealing with high-dimensional data, while being easily integrated into most posterior sampling algorithms.
As illustrated in Figure \ref{fig:eande_diagram_sampling}, given the observation $\by$, our approach requires only a single forward pass through the encoder $\hat{f}_\btheta$ to compute the low-dimensional embedding.
Subsequently, for each candidate sample $\bphi^{(s)}$ drawn from a proposal distribution, we only need to run the forward pass of the emulator branch $\hat{g}_\btheta$ to estimate its posterior probability.
In contrast, modeling $\log \hat{r}_\theta(\bphi, \by) \propto h_\theta(\bphi, \by)$ for some mapping $h_\theta$ would necessitate repeated calculations of $h_\theta$ (and hence repeated processing of the high-dimensional data $\by$) during the posterior sampling procedure.
This design is particularly beneficial when the dimensionality of the data exceeds that of the parameters, as is often the case in complex simulation-based inference tasks, where a more computationally intensive backbone is typically required to encode the high-dimensional data $\by$.
By eliminating the need for additional runs of the encoder $\hat{f}_\btheta$, our method substantially reduces the computational time for estimating posterior values, especially when a large number of samples are required for accurate posterior inference. An illustration of how \eande can be combined with acceptance-rejection sampling is provided in Appendix \ref{sec:algorithm-sampling}.

\subsection{Regularization with the intra-domain InfoNCE loss}

From a representation learning perspective, incorporating domain-specific data augmentations
can further structure the representation space in a physically meaningful way \citep{oord2018representation}.
Although augmentations do not come for free, existing knowledge in many fields of scientific simulators has provided us with a way to form these augmented views.
For example, in the study of dynamical systems where data are typically observed as time sequences, an augmented view of the data $\by$ could be represented by a shifted time sequence $\tilde{\by}$ simulated using the same parameters $\bphi$ but with a different initial condition.
We generalize this idea by defining a conditional distribution that samples an augmented view $\tilde{\by} \sim p(\tilde{\by} \mid \by)$ such that the posterior distribution given the augmented data remains invariant $p(\bphi \mid \by) \equiv p(\bphi \mid \tilde{\by})$.

We use the intra-domain InfoNCE loss, formulated as:
\begin{equation}
    L_{\mathcal{Y}\mathcal{Y}}(\hat{f}_\btheta, \hat{g}_\btheta, M) 
    := 
    -\frac{1}{M}\sum_{i=1}^M\log\left(\frac{e^{\hat{f}_\btheta(\tilde{\by}_i) \cdot \hat{f}_\btheta(\by_i)/\tau}}{\sum_{j=1}^M e^{\hat{f}_\btheta(\by_j) \cdot \hat{f}_\btheta(\by_i)/\tau}}\right),
    \label{eqn:LYY}
\end{equation}
where a batch of $M$ data-augmentation pairs $\{(\by_i, \tilde{\by}_i)\}_{i=1}^M \overset{\text{i.i.d.}}{\sim} p(\tilde{\by} \mid \by)\,p(\by)$ are sampled in a conditional way.
Similarly to before, $\tau$ is a scale hyperparameter. In our experiments, the same value of $\tau$ is applied to both inter-domain and intra-domain InfoNCE losses.
\section{Simulation-based inference using contrastive learning}
\label{sec:theory}
In this section, we provide a theoretical analysis of our \eande method, linking contrastive representation learning and simulation-based inference.
Our main result (Theorem~\ref{thm:main}) shows that, when trained using the symmetric inter-domain InfoNCE loss with batch size $M\to\infty$, our model $\hat{q}_\theta(\bphi\mid\by)$ converges to the true posterior $p(\bphi\mid\by)$. We also examine how \eande can achieve optimal data compression and deal with non-identifiable parameters.
Finally, we provide intuition for the contrastive latent space $\mathbb{S}^{n-1}$ by studying a synthetic example where contrastive learning learns to exactly reconstruct the generative process.
See Appendix~\ref{apx:proofs} for proofs and additional theoretical analysis, including a similar convergence result derived for the one-sided inter-domain InfoNCE loss (Appendix~\ref{sec:infonceproofs}, Corollary~\ref{cor:single}) and a discussion on using different priors during data generation (both training and validation) and inference (Appendix~\ref{sec:altpriors}).

\subsection{Learning the parameter posterior by optimizing the InfoNCE loss}\label{sec:maintheory}
To prove that the trained \eande model converges to the true posterior, we first assume that the true posterior $p(\bphi\mid\by)$ and likelihood $p(\by\mid\bphi)$ both belong to the exponential family parameterized by sufficiently flexible learnable embedding functions $\hat{f}_\theta: \mathcal{Y} \to \mathbb{S}^{n-1}, \hat{g}_\theta: \Phi \to \mathbb{S}^{n-1}$.

\begin{restatable}{assumption}{assumpflexiblemain}\label{assump:flexiblemain}
There exists $\theta^*$ such that both
\begin{align}
    p(\bphi\mid\by) &=  C_{\Phi\mathcal{Y}}(\by; \theta^*)^{-1}\,e^{\hat{f}_{\theta^*}(\by)\cdot \hat{g}_{\theta^*}(\bphi)/\tau}\,p(\bphi)\label{eqn:assump1}
\intertext{and}
    p(\by\mid\bphi) &=  C_{\mathcal{Y}\Phi}(\bphi; \theta^*)^{-1}\,e^{\hat{f}_{\theta^*}(\by)\cdot \hat{g}_{\theta^*}(\bphi)/\tau}\,p(\by),
\end{align}
where $C_{\Phi\mathcal{Y}}, C_{\mathcal{Y}\Phi}$ are normalization factors.
\end{restatable}

Then, we show that the symmetric inter-domain InfoNCE loss (with batch size $M\to\infty$) bounds the Kullback--Leibler (KL) divergence between the true posterior $p(\bphi\mid\by)$ and the model $\hat{q}_\theta(\bphi\mid\by)$, with $\hat{q}_{\theta^*}(\bphi\mid\by) = p(\bphi\mid\by)$ at the global minimum of the InfoNCE loss. Compared with a similar result for the one-sided inter-domain InfoNCE loss (Appendix~\ref{sec:infonceproofs}, Corollary~\ref{cor:single}), the symmetric form of the InfoNCE loss also ensures that the normalization factor $C(\by) = C^*$ of the trained model $\hat{q}_{\theta^*}(\bphi\mid\by)$ is invariant to different data instances $\by$, improving inference performance \citep{ma2018noise} and diagnostics \citep{miller2022contrastive}.

\begin{restatable}{theorem}{thmmain}\label{thm:main}
The asymptotic symmetric inter-domain InfoNCE loss
\begin{align}
    \overline L_\mathrm{sym}(\hat{f}_\theta,\hat{g}_\theta) := \lim_{M\to\infty} \left[L_\mathrm{sym}(\hat{f}_\theta,\hat{g}_\theta, M) - 2\log M\right]
\end{align}
bounds the KL divergence between the true posterior $p(\bphi\mid\by)$ and the model $\hat{q}_\theta(\bphi\mid\by)$:
\begin{align}
    \E_{\by\sim p(\by)}\left[D_\mathrm{KL}(p(\bphi\mid\by)\,\|\,\hat{q}_\theta(\bphi\mid\by))\right] &= D_\mathrm{KL}(p(\bphi,\by)\,\|\,\hat{q}_\theta(\bphi\mid\by)\,p(\by))\\
    &\le \overline L_\mathrm{sym}(\hat{f}_\theta,\hat{g}_\theta) + 2\,I(\bphi,\by),
\end{align}
where $I(\bphi,\by) := D_\mathrm{KL}(p(\bphi,\by)\,\|\,p(\bphi)\,p(\by))$ is the mutual information between $\bphi$ and $\by$.\\

\noindent Furthermore, given Assumption~\ref{assump:flexiblemain}, the global minimum is
\begin{align}
    \min_\theta \overline L_\mathrm{sym}(\hat{f}_\theta,\hat{g}_\theta) = -2\,I(\bphi,\by),
\end{align}
and, for any global minimizer $\theta^* \in \argmin_{\theta}  \overline L_\mathrm{sym}(\hat{f}_\theta,\hat{g}_\theta)$, the model posterior
\begin{align}
    \hat{q}_{\theta^*}(\bphi\mid\by) := \hat{r}_{\theta^*}(\bphi,\by)\,p(\bphi) = p(\bphi\mid\by)
\end{align}
and model likelihood-to-evidence ratio
\begin{align}
    \hat{r}_{\theta^*}(\bphi,\by) = C^{*-1}\,e^{\hat{f}_{\theta^*}(\by)\cdot \hat{g}_{\theta^*}(\bphi)/\tau} = r(\bphi,\by),
\end{align}
where $C^*$ is a normalization constant that does not vary with $\by$.
\end{restatable}

Because the trained model posterior $\hat{q}_{\theta^*}(\bphi\mid\by)$ has an exponential family form, the resulting data embedding $\hat{f}_{\theta^*}(\by)$ is a sufficient statistic for the parameters $\bphi$, which matches our intuition.

\begin{restatable}{corollary}{corsufficientstatisticmain}\label{cor:sufficientstatisticmain}
If $\by\sim G(\bphi)$, then $\hat{f}_{\theta^*}(\by)$ is a sufficient statistic for $\bphi$.
\end{restatable}

\subsection{Optimal data compression and non-identifiable parameters}
\label{sec:non-identifiable}
\eande is designed to handle high-dimensional data and non-identifiable parameters, which result in complex multimodal posteriors, by learning a summary statistic $\hat{f}_{\theta^*}$ and a latent emulator $\hat{g}_{\theta^*}$. To examine the properties of these learned embeddings, consider the following general result for any sufficient statistic $S(\by)$ and any reparameterization $\Pi(\bphi)$.

\begin{restatable}{theorem}{thmembeddings}\label{thm:embeddings}
Assuming the Theorem 1 conditions hold, for surjective maps $S: \mathcal{Y} \to \mathcal{M}$ and $\Pi: \Phi \to \Psi$ such that
\begin{align}
    p(\bphi\mid\by) &= p(\bphi\mid S(\by))\\
    p(\by\mid\bphi) &= p(\by\mid \Pi(\bphi)),
\end{align}
there exist $\hat{f}_\mathcal{M}:\mathcal{M}\to\mathbb{S}^{n-1}$ and $\hat{g}_\Psi:\Psi\to\mathbb{S}^{n-1}$ such that, for all $\by \in \mathcal{Y},\bphi \in \Phi$,
\begin{align}
    \hat{f}_\mathcal{M}(S(\by))\cdot\hat{g}_\Psi(\Pi(\bphi)) = \hat{f}_{\theta^*}(\by)\cdot\hat{g}_{\theta^*}(\bphi).
\end{align}
Therefore, the likelihood-to-evidence ratio
\begin{align}
    r(\bphi,\by) = C^{*-1}\,e^{\hat{f}_\mathcal{M}(S(\by))\cdot \hat{g}_\Psi(\Pi(\bphi))/\tau}.
\end{align}
\end{restatable}

\noindent\textbf{Optimal data compression.} To deal with high-dimensional data, \eande compresses the data using the learned summary statistic $\hat{f}_{\theta^*}$, which we previously showed is a sufficient statistic for parameter estimation (Corollary~\ref{cor:sufficientstatisticmain}). Taking $S(\by)$ to be a \emph{minimal} sufficient statistic for $\bphi$ and $\Pi(\bphi) = \bphi$ to be the identity, we can construct an embedding function $\hat{f}_\mathrm{opt}:\mathcal{Y}\to\mathbb{S}^{n-1}$ that is a minimal sufficient statistic for $\bphi$ and also a drop-in replacement for $\hat{f}_{\theta^*}$. In other words, $\hat{f}_\mathrm{opt}$ is an encoder that achieves optimal data compression for parameter estimation.

\begin{restatable}{corollary}{corminimal}\label{cor:minimal}
There exists $\hat{f}_\mathrm{opt}:\mathcal{Y}\to\mathbb{S}^{n-1}$ such that
\begin{enumerate}[(i)]
    \item for $\by \sim G(\bphi)$, $\hat{f}_\mathrm{opt}(\by)$ is a minimal sufficient statistic for $\bphi$, and
    \item the likelihood-to-evidence ratio $r(\bphi,\by) = C^{*-1}\,e^{\hat{f}_\mathrm{opt}(\by)\cdot \hat{g}_{\theta^*}(\bphi)/\tau}$.
\end{enumerate}
\end{restatable}

This illustrates how \eande can, in principle, learn an optimally compressed summary statistic that only retains features influencing parameter estimation, discarding all irrelevant information in the data. While this property is not guaranteed by our optimization objective, we can encourage stronger compression by choosing a smaller latent space dimension $n$, forcing the embedding to perform dimensionality reduction.

\vspace{5pt}
\noindent\textbf{Non-identifiable parameters.} For complex multimodal data distributions, there are often sets of parameters $\Phi_\mathrm{eq}$ that all yield identical data distributions, i.e., $G(\bphi)$ is the same for any $\bphi \in \Phi_\mathrm{eq} \subseteq \Phi$. We can partition the parameter space $\Phi$ into disjoint sets $\Phi_\mathrm{eq}^{(\boldsymbol{\psi})}$ parameterized by $\boldsymbol{\psi} \in \Psi$. For example, if we take each $\Phi_\mathrm{eq}^{(\boldsymbol{\psi})}$ to an equivalence class defined by having the same data distribution $G(\bphi)$, then $\Psi \cong \Phi / {\sim_G}$ is the quotient space after modding out by the data distribution equivalence relation $\bphi_1 \sim_G \bphi_2 \Leftrightarrow G(\bphi_1) = G(\bphi_2)$.

Taking $S(\by) = \by$ to be the identity and $\Pi(\bphi) = \boldsymbol{\psi} : \bphi \in \Phi_\mathrm{eq}^{(\boldsymbol{\psi})}$ to be the projection map $\Pi:\Phi \to \Psi$, we can construct a latent emulator $\hat{g}_\mathrm{eff}: \Phi \to \mathbb{S}^{n-1}$ that only depends on $\bphi$ through the effective parameters $\boldsymbol{\psi} = \Pi(\bphi)$ and is a drop-in replacement for $\hat{g}_{\theta^*}$. In other words, $\hat{g}_\mathrm{eff}$ deals with non-identifiable parameters by ignoring variations of $\bphi$ within the non-identifiable sets $\Phi_\mathrm{eq}^{(\boldsymbol{\psi})}$.

\begin{restatable}{corollary}{corredundant}\label{cor:redundant}
Consider any partition of parameter space
\begin{align}
    \Phi = \bigsqcup_{\boldsymbol{\psi}\in\Psi} \Phi_\mathrm{eq}^{(\boldsymbol{\psi})}
\end{align}
where $\forall\boldsymbol{\psi}\in\Psi$, $\forall\bphi_1,\bphi_2 \in \Phi_\mathrm{eq}^{(\boldsymbol{\psi})}$, the data distributions $G(\bphi_1) = G(\bphi_2)$. Let $\Pi: \Phi \to \Psi$ be the projection operator $\bphi \mapsto \boldsymbol{\psi} : \bphi \in \Phi_\mathrm{eq}^{(\boldsymbol{\psi})}$ associated with this partition. Then, there exists $\hat{g}_\mathrm{eff}: \Phi \to \mathbb{S}^{n-1}$ such that
\begin{enumerate}[(i)]
    \item  $\hat{g}_\mathrm{eff} = \hat{g}_\Psi \circ \Pi$ for some $\hat{g}_\Psi: \Psi\to\mathbb{S}^{n-1}$, and
    \item the likelihood-to-evidence ratio $r(\bphi,\by) = C^{*-1}\,e^{\hat{f}_{\theta^*}(\by)\cdot \hat{g}_\mathrm{eff}(\bphi)/\tau}$.
\end{enumerate}
\end{restatable}

As a special case, consider parameters $\bphi = (\bphi_{\mathrm{R}},\boldsymbol{\psi}) \in \Phi_\mathrm{R} \times \Psi$ that decompose into redundant parameters $\bphi_{\mathrm{R}} \in \Phi_\mathrm{R}$ that have no impact on the data generating process $G$, and effective parameters $\boldsymbol{\psi} \in \Psi$ that determine the distribution of $G$. In this setting, \eande can, in principle, learn a latent emulator $\hat{g}_\mathrm{eff}: \Phi \to \mathbb{S}^{n-1}$ that does not vary with the redundant parameters $\bphi_R$. We test this result in our experiments and find that the learned emulators do, in fact, learn to ignore redundant parameters (Figure~\ref{fig:posterior-vmf}).

\subsection{Understanding the learned latent space}
\label{sec:understanding}

\begin{figure}[t!]
    \centering
    \includegraphics[width=0.8\linewidth]{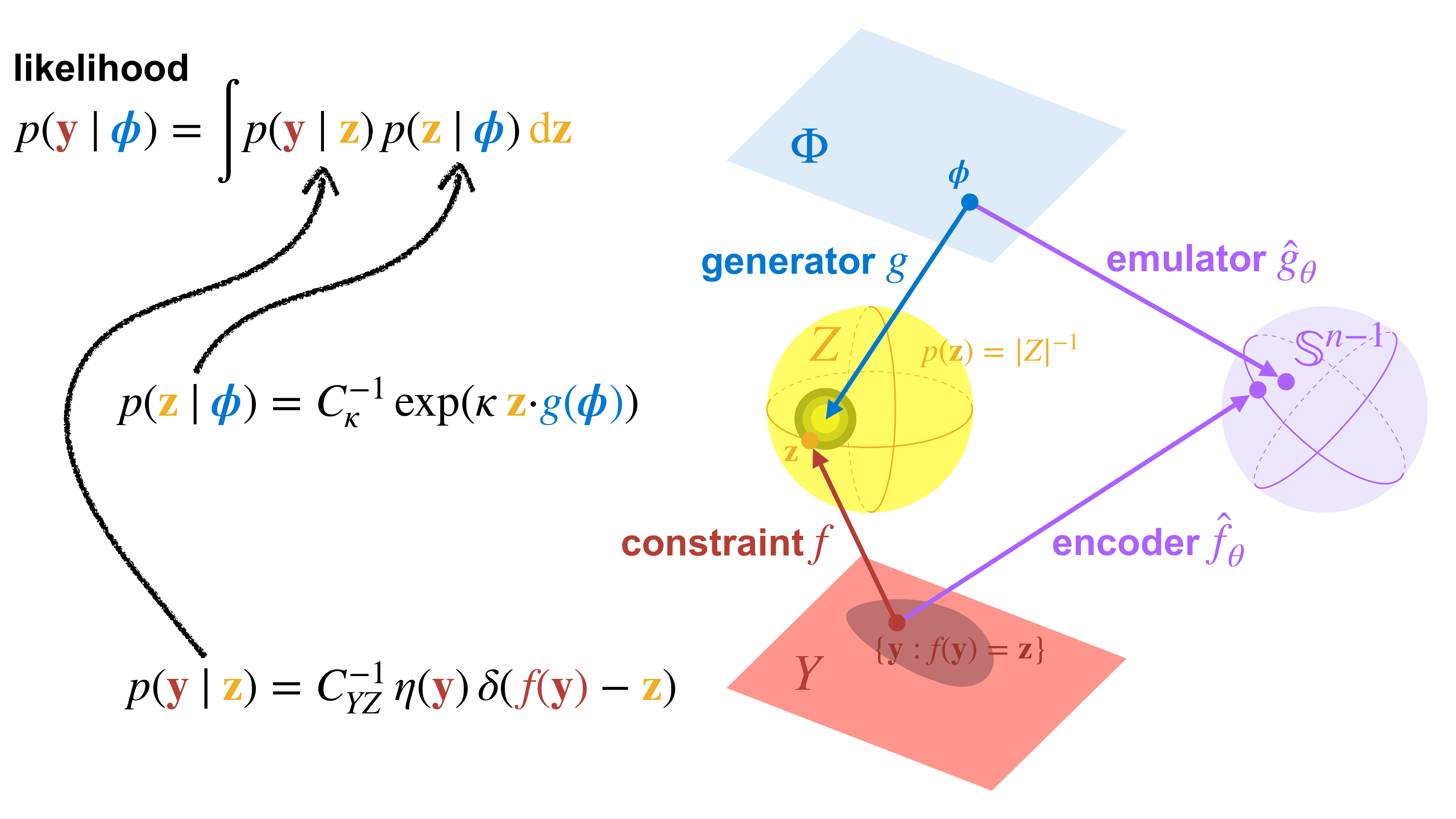}
    \caption{\textbf{Diagram of generative process model with latent space $\mathcal{Z}$ reconstructed by the learned \eande embedding space $\mathbb{S}^{n-1}$.} The generative process $p(\by\mid\bphi)$ described in Definition~\ref{def:toymodel} has a structured latent space $\mathcal{Z}$ defined by a constraint function $f:\mathcal{Y}\to\mathcal{Z}$ and a generator $g:\Phi\to\mathcal{Z}$. After training, the learned \eande embeddings $\hat{f}_\theta: \mathcal{Y}\to\mathbb{S}^{n-1}$, $\hat{g}_\theta: \Phi\to\mathbb{S}^{n-1}$ exactly reconstruct $f, g$ up to a rotation of the latent space (Theorem~\ref{thm:rotation}).}
    \label{fig:toy_task_diagram}
\end{figure}

To better understand the \eande method, we analyze an example where we can explicitly describe the discovered latent space. This example formalizes our intuition that \eande learns to reconstruct a latent generative process characterized by an embedding function $\hat{f}_{\theta^*}$, which is a sufficient statistic, and a latent emulator $\hat{g}_{\theta^*}$ for the learned statistic.

First, we define a generative process $\Phi \to \mathcal{Z} \to \mathcal{Y}$, from parameters $\Phi$ to latent space $\mathcal{Z} := \mathbb{S}^{n-1}$ to data $\mathcal{Y}$, defined in terms of a constraint function $f:\mathcal{Y}\to\mathcal{Z}$ and a generating function $g: \Phi\to\mathcal{Z}$. The generator $g$ maps the parameters $\bphi\in\Phi$ to the latent space $\mathcal{Z}$ with some von Mises--Fisher noise $\bz \sim p(\bz\mid\bphi) \propto e^{\kappa\,\bz\cdot g(\bphi)}$, and the constraint $f$ describes the manifold on which the data $\by\in\mathcal{Y}$ is sampled conditional on the latent parameter $\bz \in \mathcal{Z}$.

\begin{definition}\label{def:toymodel}
Consider a model for generating data $\by \in \mathcal{Y}$ given parameters $\bphi \in \Phi$ with an intermediate latent parameter $\mathbf{z} \in \mathcal{Z} := \mathbb{S}^{n-1}$ such that
\begin{align}\label{eqn:toymodel}
    p(\by\mid \bphi) &= \int_\mathcal{Z} p(\by\mid\bz)\,p(\bz\mid\bphi)\,\mathrm{d}\bz
\end{align}
with the following conditions:
\begin{enumerate}[(i)]
    \item $p(\bz\mid\bphi) = C_\kappa^{-1}e^{\kappa\,\bz\cdot g(\bphi)}$ is a von Mises--Fisher distribution, where $g: \Phi\to\mathcal{Z}$ is a generating function, and $C_\kappa$ is a normalization constant;
    \item $p(\by\mid\bz) = C_{\mathcal{Y}\mathcal{Z}}^{-1}\,\eta(\by)\,\delta(f(\by) - \bz)$, where $f:\mathcal{Y}\to\mathcal{Z}$ is a constraint function, $\eta(\by) \ge 0$ is a (normalized) probability distribution, $\delta(\cdot)$ is the delta distribution, and $C_{\mathcal{Y}\mathcal{Z}}$ is a normalization constant;
    \item $p(\bz) = |\mathcal{Z}|^{-1}$, i.e., the latent parameter $\bz$ has a uniform marginal distribution on the unit hypersphere $\mathbb{S}^{n-1}$ with surface area $|\mathcal{Z}| = 2\pi ^{n/2}/\Gamma(n/2)$.
\end{enumerate}
\end{definition}

Then, we show that the learned $\hat{f}_{\theta^*},\hat{g}_{\theta^*}$ and ground truth $f,g$ are equivalent, respectively, up to an arbitrary rotation of the latent space. In other words, the latent embedding space discovered by \eande is precisely the intermediate latent space $\mathcal{Z}$ which defines this generative process. $\hat{f}_{\theta^*}$ learns the sufficient statistic given by $f$, and $\hat{g}_{\theta^*}$ emulates the generating function $g$. This construction generalizes the example given in \citet{zimmermann2021contrastive} from single-domain to cross-domain contrastive learning.

\begin{restatable}{lemma}{lemnormconstant}\label{lem:normconstant}
The marginal $p(\by) = \eta(\by)$, the constant $C_{\mathcal{Y}\mathcal{Z}} = |\mathcal{Z}|^{-1}$, and the likelihood
\begin{align}
    p(\by\mid \bphi) = |\mathcal{Z}|C_\kappa^{-1}e^{\kappa\,f(\by)\cdot g(\bphi)}\,p(\by).
\end{align}
\end{restatable}

\begin{restatable}{theorem}{thmrotation}\label{thm:rotation}
Assuming the Theorem~\ref{thm:main} conditions hold and the hyperparameter $\tau = 1/\kappa$,
\begin{align}
    \hat{f}_{\theta^*}(\by) &= Rf(\by)\\
    \hat{g}_{\theta^*}(\bphi) &= Rg(\bphi)
\end{align}
for some orthogonal matrix $R \in \mathrm{SO}(n)$.
\end{restatable}

Furthermore, we can reinterpret this result by viewing the conditions given in Definition~\ref{def:toymodel} not as a specific solvable example but as describing the general class of latent spaces discovered by \eande and, more broadly, contrastive representation learning. \eande learns to decompose an arbitrary generative process into one of the form described in Definition~\ref{def:toymodel}. This form then makes it easy to derive the posterior
\begin{align}
    p(\bphi\mid\by) = \frac{p(\by \mid \bphi)\,p(\bphi)}{p(\by)} = |\mathcal{Z}|C_\kappa^{-1}e^{\kappa\,f(\by)\cdot g(\bphi)}\,p(\bphi).
\end{align}
In other words, \eande performs simulation-based inference by solving a representation learning problem.

\section{Experiments}
\label{sec:experiments}
In this section, we conduct experiments to empirically verify the performance of \eande. First, we generate synthetic datasets in both unimodal and multimodal setups, demonstrating that our method can accurately recover the true posterior. In the second part, we focus on data generated from dynamical systems. In a complex multimodal setup, we show that \eande significantly outperforms the baselines.

\subsection{Synthetic task: Reconstructing the generative process}
\label{sec:vmf}
We construct a synthetic task to illustrate our theoretical claims and our intuition for the latent space discovered by \eande. Based on the general framework described in Section~\ref{sec:understanding}, we simulate a generative process with an intermediate latent space and show that \eande learns to reconstruct the latent space and model the parameter posterior. We also augment our synthetic task to show how redundant parameters, which result in a multimodal posterior, are ignored by the learned latent emulator as described in Section~\ref{sec:non-identifiable}.

For the synthetic unimodal task, the parameters $\bphi \in \Phi = \mathbb{R}^2$ are sampled from the prior
\begin{equation}
    p(\bphi) = |\det(A)|\,p_{U(\mathbb{S}^1)}(A\bphi),
\end{equation}
where $A$ in an invertible $2\times 2$ matrix, and $p_{U(\mathbb{S}^1)} = \delta_{\mathbb{S}^1}/(2\pi)$ is the uniform distribution with support on $\mathbb{S}^1 \in \mathbb{R}^2$. Then, an intermediate latent variable $\bz \in \mathcal{Z} = \mathbb{S}^1$ is generated by a von Mises--Fisher distribution
\begin{equation} \label{eqn:synthetic-latent}
    p(\bz\mid\bphi) = C_\kappa^{-1}e^{\kappa\,\bz\cdot g(\bphi)} = C_\kappa^{-1}e^{\kappa\,\bz\cdot (A\bphi)},
\end{equation}
where $g(\bphi) := A\bphi$ is the generating function. Along with the prior $p_{U(\mathbb{S}^1)} = \delta_{\mathbb{S}^1}/(2\pi)$, this implies the latent space marginal $p(\bz) = |\mathcal{Z}|^{-1} = (2\pi)^{-1}$ is uniform as required in Section \ref{sec:understanding}. Finally, the data $\by \in \mathcal{Y} = \mathbb{R}^2$ is given by
\begin{align}
    \by = f^{-1}(\bz) = \mathrm{MLP}(\bz),
\end{align}
where MLP is an invertible multilayer perception (Appendix \ref{apx:exp-synthetic}), and $f(\by) := \mathrm{MLP}^{-1}(\by)$ is the constraint function. The marginal $p(\by) = p(\bz)\left|\det(\partial f/\partial\by)\right| = p(\bz)\left|\det(\partial\mathrm{MLP}/\partial\by)\right|^{-1}$, so the likelihood
\begin{align}
    p(\by\mid\bphi) &= \left|\det(\partial\mathrm{MLP}/\partial\by)\right|^{-1}C_\kappa^{-1}e^{\kappa\,\mathrm{MLP}^{-1}(\mathbf{y})\cdot (A\bphi)}\\
    &= 2\pi C_\kappa^{-1}e^{\kappa\,f(\by)\cdot g(\bphi)}\,p(\by)
\end{align}
is of the form given in Lemma~\ref{lem:normconstant}.

In addition, to test the effect of redundant parameters, we create a synthetic multimodal task by concatenating a redundant parameter $\bphi_R \sim U(0,1)$, with a uniform prior on the unit interval, to the original parameters $\bphi$, giving a new set of parameters $\bphi' = (\bphi_R, \bphi) \in \mathbb{R}^3$. The latent variable $\bz \sim p(\bz\mid\bphi') = p(\bz\mid\bphi) = C_\kappa^{-1}e^{\kappa\,\bz\cdot g(\bphi)}$ and the data $\by \sim p(\by\mid\bz)$ are generated in the same manner as before, so $\by$ does not depend on the redundant parameter $\bphi_R$, i.e., $p(\by\mid\bphi') = p(\by\mid\bphi)$. This allows us to verify that the latent emulator learns to ignore the redundant parameter as shown by Corollary~\ref{cor:redundant}.

\begin{figure}[t]
\centering
\noindent
\begin{minipage}{.03\textwidth}
\end{minipage}%
\begin{minipage}{1\textwidth}
  \centering
  \begin{subfigure}{.5\textwidth}
    \centering
    \includegraphics[trim={0.5cm 0.5cm 3.5cm 6.3cm},clip,height=5cm]{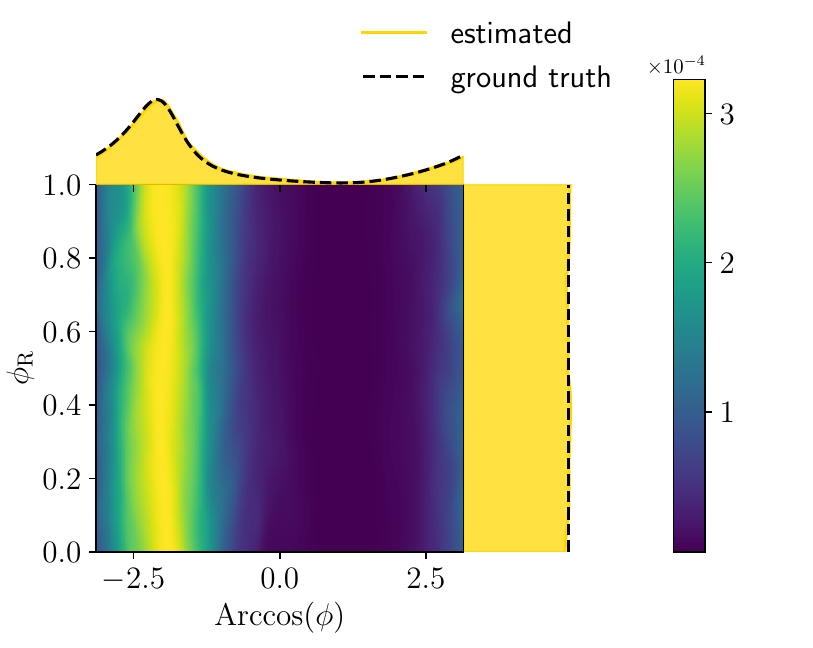}
    \caption{$p(\bphi \mid \by)$ for $\kappa = 2$}
  \end{subfigure}%
  \hfill
  \begin{subfigure}{.5\textwidth}
    \centering
    \includegraphics[trim={0.5cm 0.5cm 3.5cm 6.3cm},clip,height=5cm]{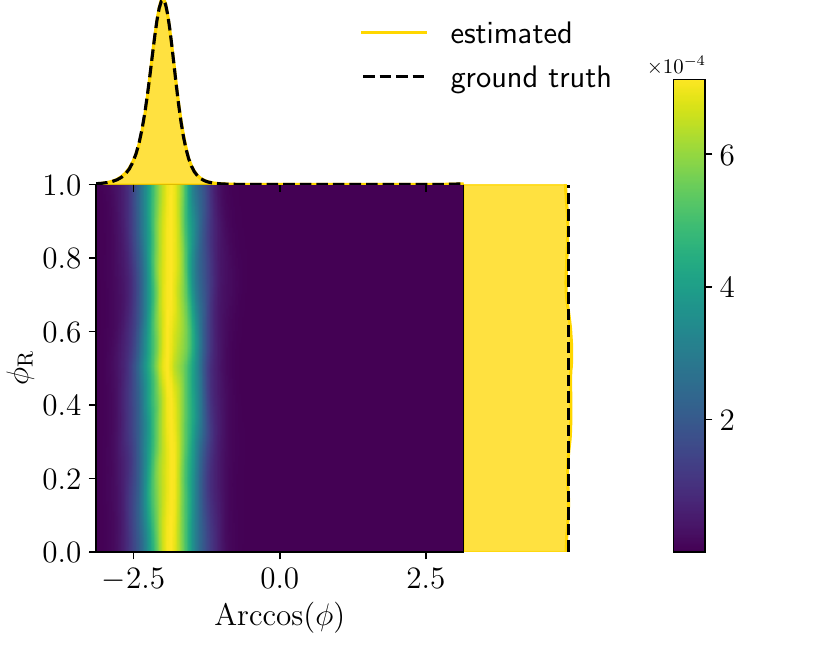}
    \caption{$p(\bphi \mid \by)$ for $\kappa = 8$}
  \end{subfigure}%
\end{minipage}
\caption{\textbf{Visualization of the estimated joint and marginal posterior distributions for one selected test sample for multimodal task.}
The results verify that the latent emulator learns to ignore irrelevant information (in the form of a redundant parameter $\bphi_\mathrm{R}$), as predicted by our theory.
}
\label{fig:posterior-vmf}
\end{figure}

\begin{table*}[t]

\subfloat[
\textbf{Unimodal synthetic data}. 
\label{tab:single}
]{
\begin{minipage}[t]{0.48\linewidth}{
\scalebox{0.82}{
\begin{tabular}{c c c c}
\toprule
$\kappa$ & {$l1\text{ }{\rm Dist.} \big(\hat{q}_\theta, p)$ $\downarrow$} & 
{$R^2(g, \hat{g}_\btheta) \uparrow$} &
{$R^2(f, \hat{f}_\btheta) \uparrow$} \\
\midrule
2 & 0.032 (0.040, 0.061) & 0.999 & 0.999
\\
8 & 0.041 (0.033, 0.085) & 0.999 & 0.999
\\
\bottomrule
\end{tabular}
}}
\end{minipage}
}
\subfloat[
\textbf{Multimodal synthetic data}. 
\label{tab:multi}
]{
\begin{minipage}[t]{0.48\linewidth}{
\scalebox{0.82}{
\begin{tabular}{c c c c}
\toprule
$\kappa$ & {$l1\text{ }{\rm Dist.} \big(\hat{q}_\theta, p)$ $\downarrow$} &
{$R^2(g, \hat{g}_\btheta) \uparrow$} &
{$R^2(f, \hat{f}_\btheta) \uparrow$} \\
\midrule
2 & 0.049 (0.044, 0.058) & 0.999 & 0.999
\\
8 & 0.055 (0.050, 0.068) & 0.999 & 0.999
\\
\bottomrule
\end{tabular}
}}
\end{minipage}
}

\caption{\textbf{Empirical Performance of \eande on synthetic tasks.} 
We present results for 50 test samples, showing the median (25th, 75th percentile) of the $l1$ distance between the estimated and ground truth posteriors, along with the $R^2$ score.
The results demonstrate that \eande accurately models the posterior in both unimodal and multi-modal cases, with high $R^2$ scores (close to 1) indicating a strong linear relationship between the learned embeddings and the true latent variables.
}
\label{table:synthentic}
\end{table*}

\vspace{5pt}
\noindent \textbf{Evaluation.} We evaluate both the quality of the posterior and the role of the embedding.
We compute the posterior of \eande constructed using embeddings in the form of Equation \ref{eqn:posterior}, where we estimate the normalization factor
\begin{align}
    C(\by) \approx \sum_{i=1}^{N'} e^{\hat{f}_\btheta(\by)\cdot \hat{g}_\btheta(\bphi_i)/\tau}
\end{align}
using $N'=10000$ sampled parameters drawn from the prior $\{\bphi_i\}_{i=1}^{N'} \overset{i.i.d.}{\sim} p(\bphi)$.
For direct evaluation of the estimated posterior, we compute the sum of $l1$ distance between the estimates $\hat{q} (\bphi \mid \by)$ and the true posterior $p(\bphi \mid \by)$ over these $N'=10000$ sampled parameters (Appendix \ref{apx:exp-synthetic}).
To test whether the embedding of \eande successfully recovers the latent space defined in the generative process in Theorem~\ref{thm:rotation}, we fit a linear regression between the source signal and its corresponding embedding respectively for 
$\big(g(\bphi), \hat{g}_\btheta(\bphi)\big)$ and $\big(f(\by), \hat{f}_\btheta(\by) \big)$ pairs, and evaluate the quality of linear regression using the $R^2$ statistic.

\vspace{5pt}
\noindent \textbf{Results.} As shown in Table \ref{table:synthentic}, \eande effectively captures the true posterior in both unimodal and multimodal scenarios, achieving an $l1$ distance close to zero over 10000 sampled parameters. The embedding learned by \eande reconstructs the latent space as specified by the generative process (Theorem~\ref{thm:rotation}): with a high $R^2$ score, there is a strong linear correlation between the pairs $\big(g(\bphi), \hat{g}_\btheta(\bphi)\big)$ and $\big(f(\by), \hat{f}_\btheta(\by) \big)$. 
Furthermore, Figure \ref{fig:posterior-vmf} confirms that the emulator $\hat{g}_\btheta$ of \eande is capable of disregarding redundant parameters, as discussed in Section~\ref{sec:non-identifiable}.
\subsection{Lorenz 96: High-dimensional dynamics data with a multimodal posterior}
\label{sec:l96-main}
In this section, we evaluate \eande in a realistic setting with high-dimensional data from a dynamical system and a multimodal posterior due to parameter redundancy.
Our setup is inspired by real data in practice 
\citep{mcguffie2001forty, neelin2010considerations, sexton2012multivariate}.

We conduct a numerical case study using the Lorenz 96 (L96) model, a prototype model for climate science and geophysical applications.
The key parameter that controls the dynamics of the L96 model is the forcing term $F$, which determines the bifurcation behavior of the chaotic dynamics \citep{kerin2020lorenz}. The governing equations for L96 are
\begin{eqnarray*}
    \odv{\mathbf{u}^k_t}{t} = -\mathbf{u}^{k-1}_t(\mathbf{u}^{k-2}_t - \mathbf{u}^{k+1}_t) - \mathbf{u}^k_t + F - c \bar{\mathbf{v}}^k_t, \\
    \frac{1}{c} \odv{\mathbf{v}^{j,k}_t}{t} = -\mathbf{v}^{j+1,k}_t(\mathbf{v}^{j+2,k}_t - \mathbf{v}^{j-1,k}_t) - \mathbf{v}^{j,k}_t + \frac{1}{J}\mathbf{u}^k_t,
\end{eqnarray*} 
where $\mathbf{u} \in \mathbb{R}^{T \times K}$ denotes the slow variable, with the subscript $t$ indicating the $t$-th timestamp, and $\mathbf{v} \in \mathbb{R}^{T \times KJ}$ denotes the fast variable, with $\bar{\mathbf{v}}^k_t = \frac{1}{J} \sum_{j=1}^J \mathbf{v}^{j,k}_t$.
We use $\mathbf{y}: = [\mathbf{u}, \mathbf{v}] \in \mathbb{R}^{T \times K(J+1)}$ to denote the time sequence observed in the system over a period of duration $T$.
We set $c=10$ to maintain that the simulated data remains within the chaotic regime, and $K=36$ and $J=10$, following \citet{schneider2017earth}.
When T=250 (as in our example below), y is $9.9\times10^4$ dimensional---a much higher dimensionality than is commonly used in SBI settings.

To simulate realistic scenarios where the input parameters of the generative process are interrelated, leading to intricate multimodal posterior distributions, we introduce two parameters $F_1,F_2$ that relate to the forcing term by $F^2 = F_1^2+F_2^2$. In our experiments, 
we mimic a scenario where the simulator $G$ takes in parameters $F_1$ and $F_2$ (e.g., as if the simulator designer were unaware of the simpler parameterization in terms of $F$ alone). Hence there are many $\bphi = (F_1,F_2)$ pairs that produce equivalent outputs of $G$. Our goal is to estimate the posterior over parameters $F_1$ and $F_2$ and for that estimated posterior to accurately reflect this nonidentifiability. Note that a uniform distribution over $F$ would correspond to a uniform distribution over a circle in $(F_1,F_2)$ space with radius $F$, and so an accurate method would produce a posterior with mass concentrated uniformly along this circle.

\begin{figure}[t]
\centering
\begin{minipage}{1\textwidth}
  \centering
  \begin{subfigure}{.3\textwidth}
    \centering
    \includegraphics[trim={0.3cm 0.5cm 7cm 6.6cm},clip,height=4.5cm]{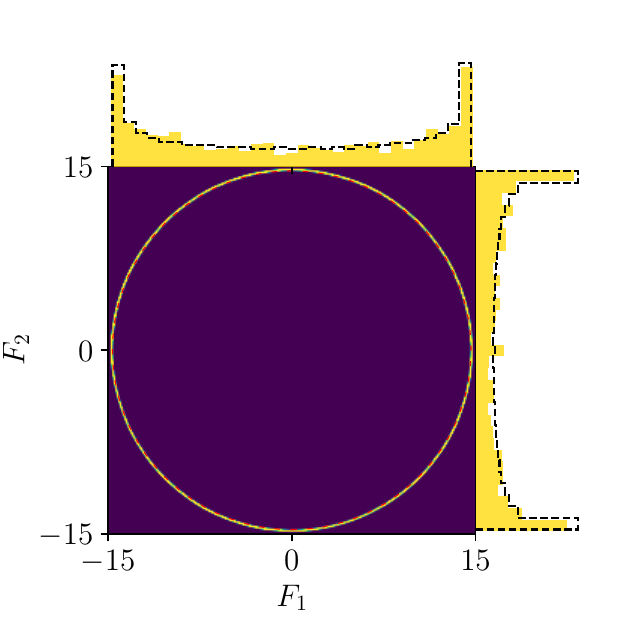}
    \caption{\eande}
  \end{subfigure}%
  \hfill
  \begin{subfigure}{.3\textwidth}
    \centering
    \includegraphics[trim={0.3cm 0.5cm 7cm 6.6cm},clip,height=4.5cm]
    {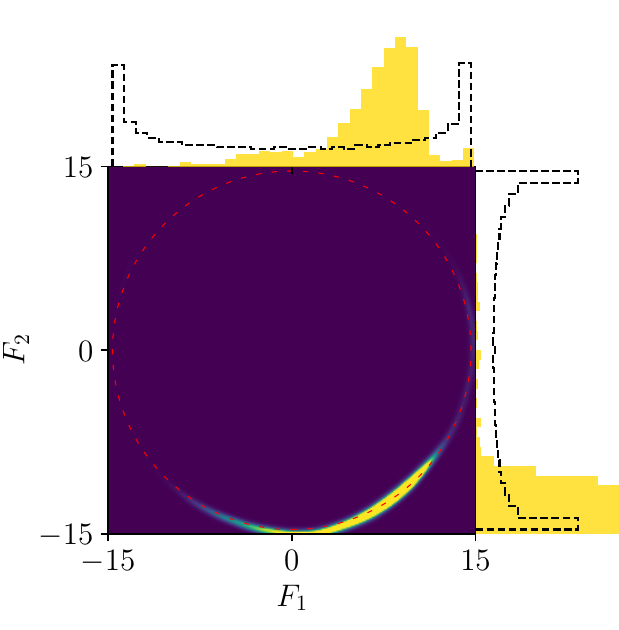}
    \caption{NRE-C}
  \end{subfigure}%
  \hfill
  \begin{subfigure}{.37\textwidth}
    \centering
    \includegraphics[trim={0.3cm 0.5cm 3cm 6.6cm},clip, height=4.5cm]
    {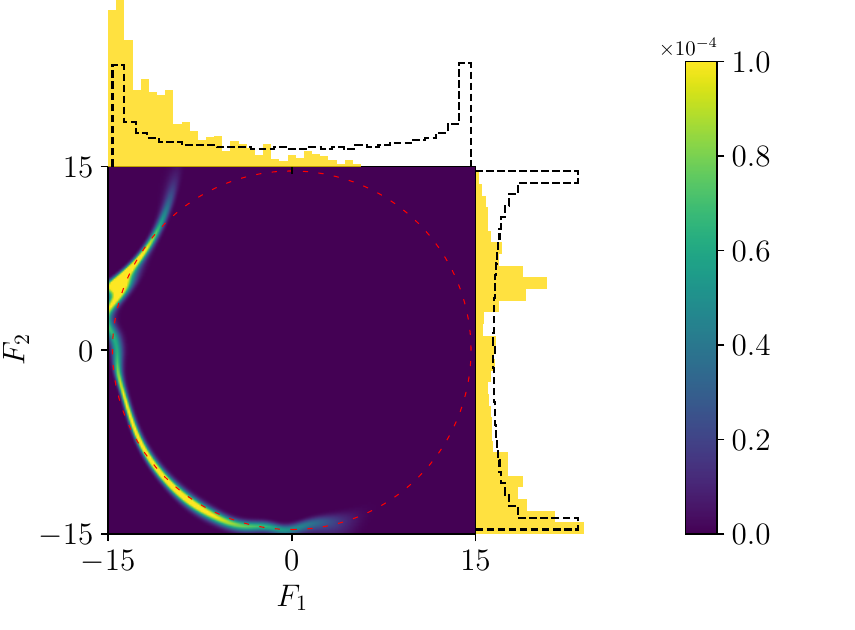}
    \caption{NPE-C}
  \end{subfigure}
\end{minipage}
\caption{\textbf{Visual comparison of the estimated joint and marginal posterior distributions for one test sample.} 
In each subplot, the heatmap displays the estimated posterior probabilities (with the maximum value clipped for better visualization), and the red dashed circle represents the ground truth reference distribution.
The histograms showing marginal densities in the upper and right portions of each subplot are plotted using samples drawn from the estimated posterior using the acceptance-rejection sampling, with the dashed black line illustrating the histograms of samples from the reference distribution.
The results illustrate that \eande captures the full spread of the posterior, whereas NRE-C \citep{miller2022contrastive} and NPE-C \citep{greenberg2019automatic} concentrate on a limited region of the circle, resulting in skewed estimates.
}
\label{fig:posterior}
\end{figure}

\vspace{5pt}
\noindent\textbf{Data generation.} We set the prior distribution $p(\bphi)$ as a two-dimensional uniform distribution over the square region $[-15, 15] \times [-15, 15]$. 
To generate training data, we draw 500 samples $\{\bphi_i\}_{i=1}^N$ from the prior $p(\bphi)$, and then simulate dynamical data using a numerical ODE solver up to $\widetilde{T} = 2000$ starting from random initial conditions sampled from standard normal distributions. This defines the generative process $\by \sim G(\bphi)$. The data is then cropped to a random interval of length $T=250$ during training.  Each $\{\bphi^o_i\}_{i=1}^{N'}$ in the test set is sampled from same prior distribution $p(\bphi)$. At test time, given an observation $\by^o \sim G(\bphi^o)$ with length $T=250$, we aim to estimate the parameter posterior $p(\bphi\mid\by^o)$. 

\vspace{5pt}
\noindent\textbf{Ground truth reference distribution.} While we do not have access to the true posterior in this setting, we can construct a reference distribution $p_\mathrm{ref}(\bphi\mid\by^o)$ that is simply the posterior under the assumption that the original forcing parameter $F$ is always uniquely identifiable from the data $\by^o$. Recalling that $\bphi = (F_1, F_2)$ is redundantly parameterized such that $F^2 = F_1^2 + F_2^2$, the likelihood $p(\by^o\mid \bphi)$ must be invariant to rotations of $\bphi$ about the origin. Thus, given a particular $\bphi^o$ which generates $\by^o\sim G(\bphi^o)$, we can define a set of equivalent parameters $\Phi^o:=\{\bphi: \bphi \in \mathrm{supp}(p(\bphi)), \bphi = R\bphi^o, R\in\mathrm{SO}(2)\}$ that includes rotations of $\bphi$, all of which must have equal likelihood $p(\by^o\mid \bphi)$ and therefore equal posterior $p(\bphi\mid\by^o)$ given a uniform prior $p(\bphi)$. Assuming $F$ is identifiable, the only uncertainty in the posterior comes from this rotational redundancy, so the reference distribution $p_\mathrm{ref} = p_{U(\Phi^o)}$ is the uniform distribution with support on $\Phi^o$. Even if $F$ is not strictly identifiable, as long as the original parameter posterior $p(F\mid\by^o)$ consists of a single sharply peaked mode, this reference distribution remains a good choice for performance comparisons.

\begin{figure}
    \centering
    \includegraphics[width=0.5\linewidth]{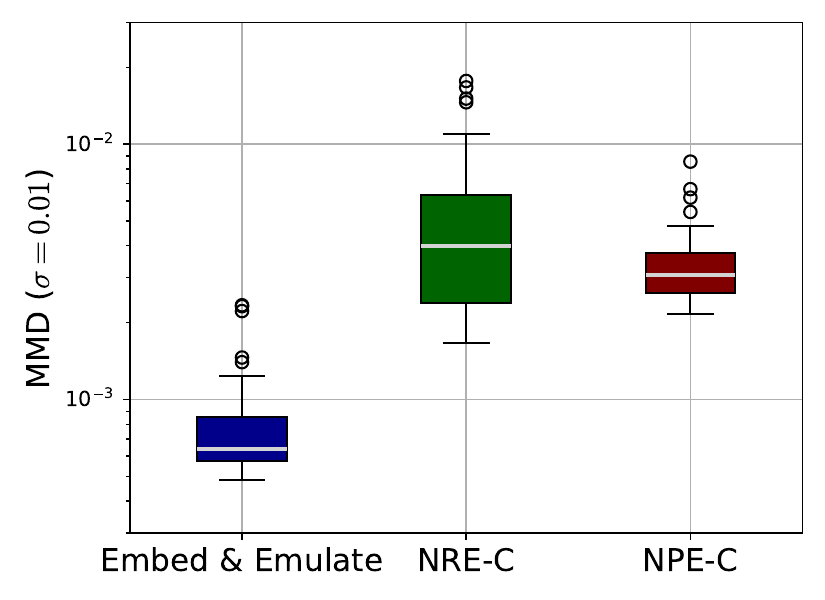}
    \caption{\textbf{Comparison of samples quality using maximum mean discrepancy over 50 testing instances (MMD).} Each box plot shows the median (25th, 75th percentiles) of the error statistics.
    We compare \eande with NRE-C \citep{ miller2022contrastive} and NPE-C \citep{greenberg2019automatic}.
    The results demonstrate that \eande achieves significantly lower errors, with a substantially reduced variance in error statistics, indicating more consistent and reliable performance.}
    \label{fig:mmd}
\vspace{-3pt}
\end{figure}
\begin{table}[t!]
    \centering
    \begin{tabular}{c c c}
       \toprule
         &  \textbf{time} & \textbf{$\#$parameters} \\
        \midrule
       \eandelong & \textbf{0.10} & 24.71M \\
        NRE-C & 3.51 & 22.86M \\
        NPE-C &  4.20 & 25.50M \\
        \bottomrule
    \end{tabular}
    \caption{\textbf{Comparison of the computational time for posterior inference in seconds.}
For each test case, we compute the total wall-clock time for posterior inference over $10,000$ different parameters $\bphi$. The result represents the average time for 50 different test instances. The results indicate that \eande achieves significant efficiency improvements over both NPE-C \citep{greenberg2019automatic} and NRE-C \citep{miller2022contrastive}. All methods use the same amount of time for data generation and similar amounts of time for training.}
\label{table:time}
\end{table}
\vspace{5pt}
\noindent \textbf{Results.} We use both the inter-domain and intra-domain InfoNCE losses for \eande (Algorithm \ref{alg:intra-regularization}), and compared its performance with NPE-C \citep{greenberg2019automatic} and NRE-C \citep{miller2022contrastive}. Both NPE-C and NRE-C can support posterior estimation of high-dimensional data with an embedding network implicitly plugged into their training pipeline. To ensure a fair comparison, we use the same backbone, i.e., ResNet34 \citep{he2016deep} for all models.

We evaluate all models using acceptance-rejection sampling (Section~\ref{sec:algorithm-sampling}), and draw 100 samples for posterior estimation per test instance. 
To reflect the quality of the posterior, we use samples to compute the maximum mean discrepancy (MMD, defined in Section~\ref{apx:l96}) using a Gaussian kernel with sigma = 0.01 between the posteriors of the learned model and the true reference distribution.
As shown in Figure~\ref{fig:mmd},
\eande demonstrates substantially lower errors than both NRE-C and NPE-C, accompanied by a significantly reduced variance in the error metric, suggesting a more robust and consistent posterior estimate. This advantage of \eande can be further confirmed by the posterior visualization of a representative test instance in Figure~\ref{fig:posterior}. Compared with the NRE-C and NPE-C posteriors, the \eande posterior is a much closer match to the reference distribution, which is a uniform distribution with support on a circle. Unlike NRE-C and NPE-C, \eande also successfully captures the rotation symmetry of the posterior given the same limited set of training simulations. We provide more visualizations for comparison in Section~\ref{apx:visualizations}.

We also compare the wall-clock time needed to compute the estimated posterior values for $10,000$ different parameters $\bphi$. Although all models facilitate parallel forward computations, memory limitations restrict the batch size per forward pass. To ensure a fair comparison, we employ a grid search to identify the maximum batch size that each model can accommodate during evaluation. As illustrated in Table \ref{table:time}, when the backbone sizes for all models are comparable in terms of the number of backbone parameters, \eande greatly cuts down the computational time compared to both NRE-C and NPE-C because, as explained in Section \ref{sec:inter-modal} and Appendix \ref{sec:algorithm-sampling}, \eande only needs to compute the data embedding once for posterior inference of each observation.

\subsection{Ablation study on the intra-domain InfoNCE loss}
\label{sec:ablation_intra}
Following the configuration described in Section \ref{sec:l96-main}, we perform an ablation study to evaluate the benefits introduced by the intra-domain InfoNCE loss (Equation \ref{eqn:LYY}). In this experiment, the Lorenz 96 process is chaotic.
The intra-domain contrastive loss helps ensure two $\mathbf{y}$’s that may be far apart in Euclidean distance but which correspond to the same parameter values, just different initial conditions, are mapped to the same location in embedding space, as detailed in \citet{jiang2024training}. The lower variance as depicted in Figure \ref{fig:mmd-ablation}(a) associated with the intra-domain contrastive loss may be attributed to this phenomenon.

\subsection{Ablation study on the symmetric InfoNCE loss}
\label{sec:ablation_symmetry}
\begin{figure}[t!]
\begin{minipage}{1\textwidth}
    \centering
    \begin{subfigure}{.4\textwidth}
    \centering
    \includegraphics[height=5.5cm]{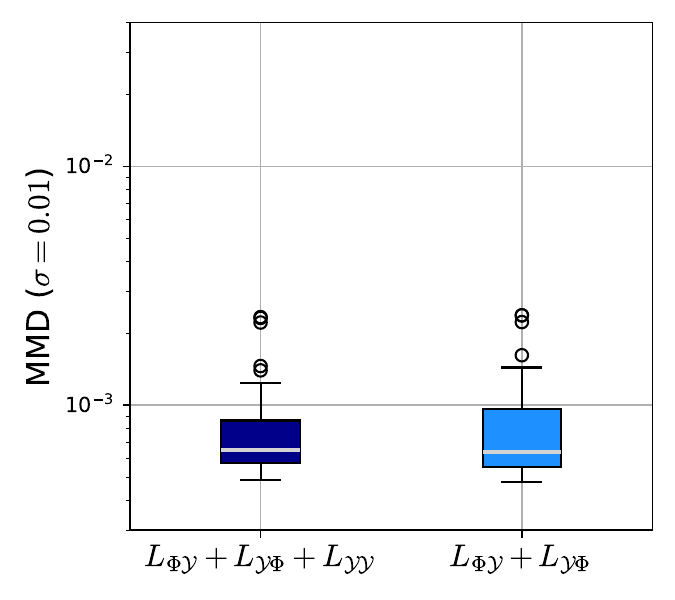}
    \caption{Intra-domain regularization.}
    \end{subfigure}
    \hfill
    \hspace{-10pt}
    \begin{subfigure}{.6\textwidth}
    \centering
    \includegraphics[height=5.5cm]{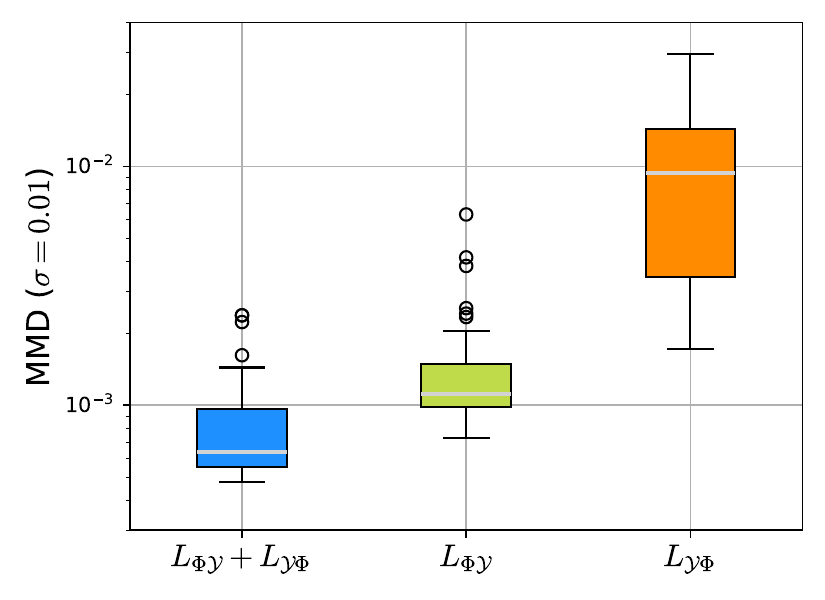}
    \caption{Symmetry of the objective.}
    \end{subfigure}
    \caption{\textbf{Ablation on the roles of different objectives using maximum mean discrepancy (MMD).} 
    We show that (a) including the intra-domain $L_{\mathcal{Y}\mathcal{Y}}$ does not lead to an appreciable reduction in the average MMD, but it does reduce the variance of the errors (Section \ref{sec:ablation_intra}), and (b) using both $L_{\Phi \mathcal{Y}}$ and $L_{\mathcal{Y} \Phi}$ yields better performance than either alone (Section \ref{sec:ablation_symmetry}).
    }
    \label{fig:mmd-ablation}
\end{minipage}
\end{figure}
\begin{figure}[t!]
\centering
\begin{minipage}{1\textwidth}
  \centering
  \begin{subfigure}{.3\textwidth}
    \centering
    \includegraphics[trim={0.3cm 0.5cm 7cm 6.6cm},clip,height=4.5cm]{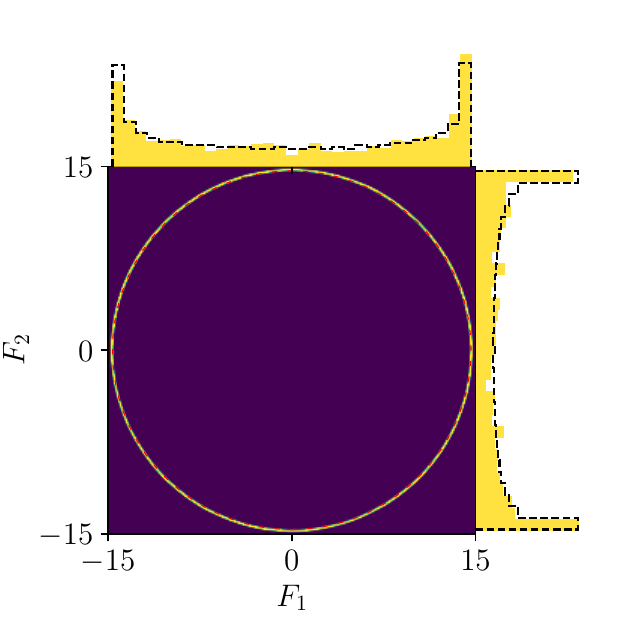}
    \caption{$L_{\Phi \mathcal{Y}} + L_{\mathcal{Y} \Phi}$}
  \end{subfigure}%
  \hfill
  \begin{subfigure}{.3\textwidth}
    \centering
    \includegraphics[trim={0.3cm 0.5cm 7cm 6.6cm},clip,height=4.5cm]
    {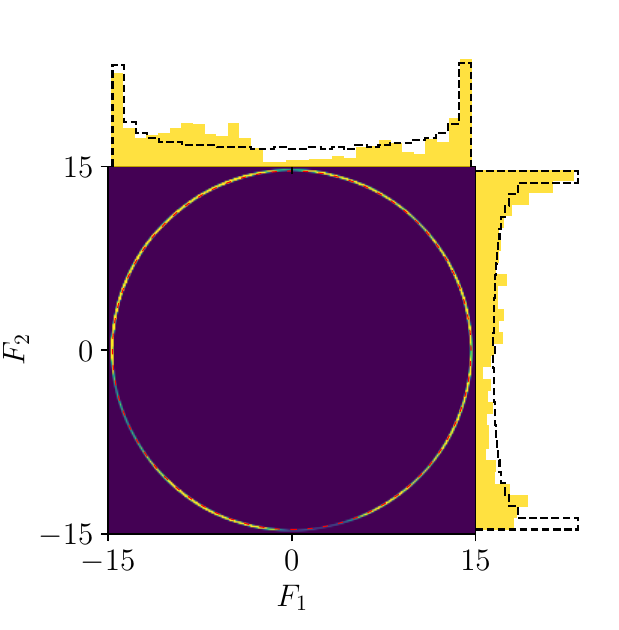}
    \caption{$L_{\Phi \mathcal{Y}}$}
  \end{subfigure}%
  \hfill
  \begin{subfigure}{.37\textwidth}
    \centering
    \includegraphics[trim={0.3cm 0.5cm 3cm 6.6cm},clip, height=4.5cm]{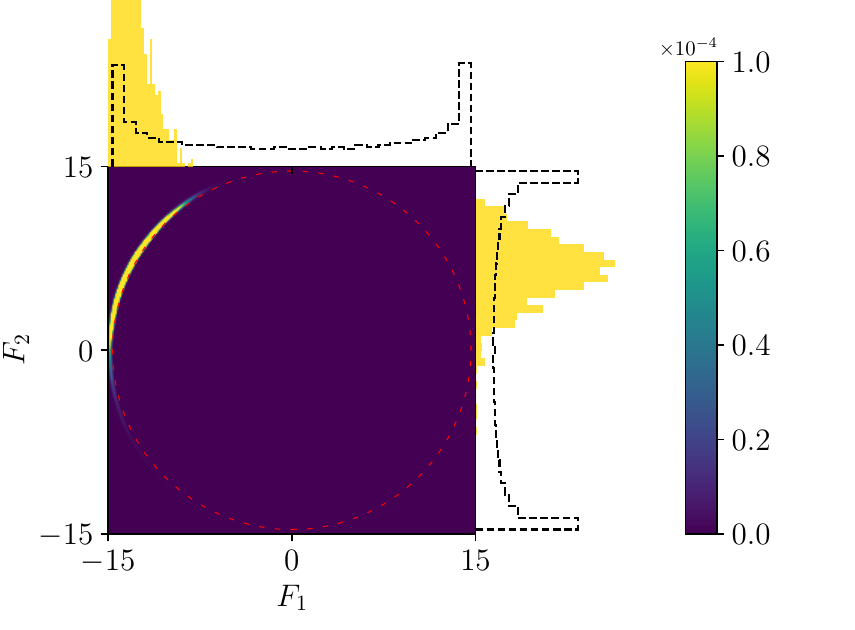}
    \caption{$L_{\mathcal{Y} \Phi}$}
  \end{subfigure}
\end{minipage}
\caption{\textbf{Visualization of the estimated joint and marginal posterior distributions using different training objectives.}
On the \textit{left}, the estimated posterior using symmetric objectives correctly captures the uniform circular shape of the ground truth reference distribution, represented by the red dashed circle.
In the \textit{middle}, the estimated posterior using only $L_{\Phi \mathcal{Y}}$ roughly captures the right shape but deviates from the uniform circular reference.
On the \textit{right}, the estimated posterior using $L_{\mathcal{Y} \Phi}$ performs the worst, focusing only on the upper right corner of the circle and giving the most biased estimates.
}
\label{fig:posterior-symmetry}
\end{figure}
\begin{figure}
    \centering
    \includegraphics[width=0.7\linewidth]{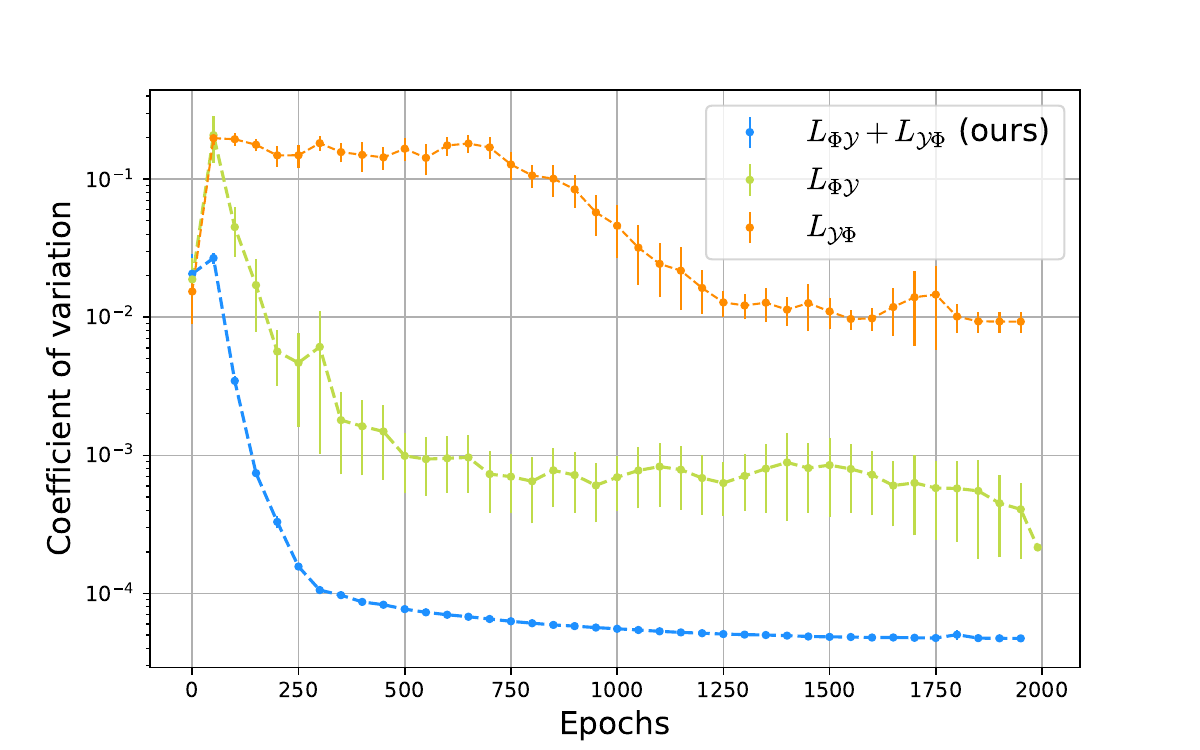}
    \caption{\textbf{Training dynamics of the coefficient of variation (CV) of the posterior normalizer, using different objectives.}
    }
    \label{fig:normalizer}
\end{figure}
We also empirically study the effect of using the symmetric version of the inter-domain InfoNCE loss in \eande. Theoretically, optimizing either the symmetric InfoNCE loss $L_\mathrm{sym} = L_{\Phi\mathcal{Y}} + L_{\mathcal{Y}\Phi}$ or the one-sided InfoNCE loss $L_{\Phi\mathcal{Y}}$ both lead to a trained model that matches the true posterior (Section~\ref{sec:maintheory}). However, as illustrated in Figure~\ref{fig:mmd-ablation}(b), the symmetric loss yields more stable and precise posterior estimations. This is further supported by the visualization in Figure~\ref{fig:posterior-symmetry}, where the estimator employing the symmetric loss better matches the uniform circular reference distribution.

As before, we follow the setup outlined in Section~\ref{sec:l96-main}. To ensure a fair comparison and gain a better understanding of the inter-domain losses, we perform these experiments without the intra-domain InfoNCE loss. For completeness, we also show experiments using the $L_{\mathcal{Y}\Phi}$ loss even though this loss is not theoretically guaranteed to converge to the true posterior.

\vspace{5pt}
\noindent \textbf{Evaluation.} We further examine the behavior of the normalization constant $C(\by)$ for the posterior estimator using various losses. For each test instance $\by^o$, we estimate its normalization constant using a Monte Carlo approach: $C(\by^o) = \sum_{i=1}^{N'} \hat{r}(\bphi_i, \by^o)$, where $N' = 10,000$ different $\bphi_i$ were sampled from the prior distribution. We then calculate the mean $\mu(C)$ and standard deviation $\sigma(C)$ of the normalization constant for a set of different test instances $\{C(\by^o_i)\}_{i=1}^{50}$. As discussed in Section \ref{sec:maintheory}, the symmetry of our InfoNCE loss should theoretically ensure that the normalization constant $C(\by)$ does not vary with different data $\by$. To verify this, we compute the coefficient of variation metric $\mathrm{CV}(C) := \sigma(C) / \mu(C)$ to measure the degree of variation of the normalization constant across different test instances.

\vspace{5pt}
\noindent
\textbf{Results.}
Ideally, when the normalization constant remains unchanged for different $\by$, the coefficient of variation should approach zero. In Figure \ref{fig:normalizer}, we depict the dynamics of the variation of the normalization constant over $2000$ epochs and replicate the experiments for 10 trials to illustrate the error bars. The results demonstrate that the symmetric loss not only results in nearly zero variation of the normalization constant at the end of training but also causes the variance of the constant to decrease much faster, with narrower error bars during the initial training epochs. This suggests that the symmetric loss yield much more stable normalization constants for the estimated posterior for different data $\by$, further supporting our theoretical analysis.

\section{Discussion}
Based on tools from contrastive representation learning, we have proposed, theoretically analyzed, and empirically tested the new \textit{Embed and Emulate} (\eande) method for simulation-based inference (SBI). This approach efficiently handles high-dimensional data by simultaneously training an encoder to learn a compressed summary statistic and a fast latent emulator to learn the mapping from the parameters to the summary statistic. By jointly training the encoder $\hat{f}_\theta$ and emulator $\hat{g}_\theta$ using the symmetric InfoNCE loss from contrastive learning, \eande not only learns a summary statistic that captures all the necessary information for parameter estimation---a sufficient statistic---but also ensures that the likelihood for the learned statistic is a simple distribution given the trained emulator (Section \ref{sec:theory}). Alternatively, we can interpret the \eande method as identifying a structured latent space via contrastive representation learning that recasts the generative process in a simple form (Section \ref{sec:understanding}).

Compared with SBI methods that directly try to learn the posterior, such as neural posterior estimation (NPE) \citep{papamakarios2016fast,lueckmann2017flexible,greenberg2019automatic,rodrigues2021hnpe,ward2022robust}, \eande benefits significantly from learning a compressed summary statistic that simplifies the likelihood model leading to faster and more sample-efficient inference for high-dimensional data. Compared with SBI methods that separate the summary statistic from the likelihood model \citep{papamakarios2019sequential}---i.e., first identifying a summary statistic and then fitting a model for the likelihood---\eande's joint training ensures that the learned statistic is both sufficient for parameter estimation and easy to model, avoiding the need for complex density estimation methods such as deep generative models.

\eande has similarities with neural ratio estimation (NRE) methods \citep{cranmer2015approximating, thomas2020likelihoodfreeinferenceratioestimation,moustakides2019training, hermans2020likelihood, miller2021truncated, miller2022contrastive}, which also parameterize the likelihood-to-evidence ratio $\hat{r}_\theta(\bphi,\by)$ and train using cross-entropy losses similar to InfoNCE. Unlike prior NRE approaches, \eande parameterizes the likelihood-to-evidence ratio $\hat{r}_\theta(\bphi,\by) \propto e^{\hat{f}_\theta(\by)\cdot\hat{g}_\theta(\bphi)/\tau}$ in terms of a similarity measure in a latent space $\mathbb{S}^{n-1}$ and also uses a symmetric InfoNCE loss rather than a one-sided cross-entropy. The \eande parameterization has several benefits over a generic parameterization of the likelihood-to-evidence ratio, including faster inference due to only needing to run the encoder once on the high-dimensional data (Section \ref{sec:l96-main}). Working with an explicit low-dimensional summary statistic also provides additional inductive bias for the structure of the generative process (Section \ref{sec:understanding}), improving sample efficiency and perhaps providing a degree of interpretability for generative processes that have a low-dimensional sufficient statistic. Relatedly, the \eande parameterization implies that the posterior has the form of an exponential family distribution. To handle posteriors with heavy tails that lack a low-dimensional sufficient statistic, we would need to adapt \eande in the future to provide a distinct inductive bias for the form of the posterior while retaining the benefits of a fast latent emulator, e.g., by changing the similarity measure and latent space that parameterize the likelihood-to-evidence ratio. The symmetric form of the loss used in \eande also has significant benefits in terms of empirical performance (Section \ref{sec:ablation_symmetry} and \citet{ma2018noise}) as well as diagnostics (Figure \ref{fig:normalizer} and \citet{miller2022contrastive}).

The structure of the latent summary statistic space, which we interpret as a learned decomposition of the generative process (Section \ref{sec:understanding}), is loosely analogous with other representation learning methods, such as variational autoencoders (VAEs) \citep{kingma2019introduction, khemakhem2020variational}, which learn a structured latent space. We speculate that, by turning SBI into a representation learning problem, we have also opened the door for more interpretability. For example, an analysis of the learned summary statistic may point to relevant interpretable features that control the data generation process. This kind of interpretability may also lead to improved generalization and a better scientific understanding of the physical processes that produced the data. We hope to further explore this direction in future work.

\section{Acknowledgements and Disclosure of Funding}
This work was supported by DOE grant DE-SC0022232 and AFOSR grant FA9550-18-1-0166.
Additional support was provided by the National Institute for Mathematics and Theory in Biology (Simons Foundations award MP-TMPS-00005320 and National Science Foundation award DMS-2235451).
Peter Y. Lu gratefully acknowledges the support of the Eric and Wendy Schmidt AI in Science Postdoctoral Fellowship, a Schmidt Sciences program.

\newpage
\appendix
\section{Algorithms}
In this section, we present pseudocodes for the learning algorithms of \eande and include an example demonstrating how to integrate \eande into acceptance-rejection sampling.
\subsection{Main training algorithm}
\label{sec:algorithm-main}
\begin{algorithm}[ht]
\caption{\eande main learning algorithm}
\renewcommand{\algorithmicrequire}{\textbf{Input:}}
\renewcommand{\algorithmicensure}{\textbf{Output:}}
\newcommand{\INDSTATE}[1][1]{\hspace{#1\algorithmicindent}}
\newcommand{\GreyComment}[1]{\textcolor{gray}{\footnotesize{$\triangleright$ \texttt{#1}}}}
\begin{algorithmic}[1]
\REQUIRE Batch size $M$, constant $\tau$. \\ 
\INDSTATE[-2] \textbf{Initialize:} Encoder $\hat{f}_\btheta$ and emulator $\hat{g}_\btheta$.
\FOR{i.i.d.~sampled batch $\{\bphi_i, \by_i\}_{i=1}^M \sim p(\bphi,\by)$}
\STATE Compute the inter-domain loss $L_\mathrm{sym}
   (\hat{f}_\theta, \hat{g}_\theta,M)$ given in (\ref{eqn:L_inter}).
\STATE Update $\hat{f}_\btheta$ and $\hat{g}_\btheta$ using the gradient $\nabla_\theta L_\mathrm{sym}
   (\hat{f}_\btheta, \hat{g}_\btheta,M)$.
\ENDFOR
\ENSURE $\hat{f}_\btheta$, $\hat{g}_\btheta$
\end{algorithmic}
\label{alg:main}
\end{algorithm}
\noindent As discussed in Section~\ref{sec:inter-modal}, the main learning algorithm of \eande jointly optimizes the weights of the encoder $f_\btheta$ and the emulator $\hat{g}_\btheta$, as illustrated in Algorithm~\ref{alg:main}.

\subsection{Training with intra-domain regularization}
\label{sec:algorithm-intra}
\begin{algorithm}[ht]
\caption{\eande learning algorithm with intra-domain regularization}
\renewcommand{\algorithmicrequire}{\textbf{Input:}}
\renewcommand{\algorithmicensure}{\textbf{Output:}}
\newcommand{\INDSTATE}[1][1]{\hspace{#1\algorithmicindent}}
\newcommand{\GreyComment}[1]{\textcolor{gray}{\footnotesize{$\triangleright$ \texttt{#1}}}}
\begin{algorithmic}[1]
\REQUIRE Batch size $M$, constant $\tau$, access to conditional distribution $p(\tilde{\by} \mid \by)$, weight $\lambda$. \\
\INDSTATE[-2] \textbf{Initialize:} Encoder $\hat{f}_\btheta$ and emulator $\hat{g}_\btheta$.
\FOR{sampled batch $\{\bphi_i, \by_i\}_{i=1}^M$}
\FOR{for all $i \in \{1,\dots, M\}$}
\STATE Draw augmented view $\tilde{\by}_i$ from the conditional distribution $p(\tilde{\by} \mid \by)$
\ENDFOR
\STATE Compute the inter-domain loss $L_\mathrm{sym}
   (\hat{f}_\theta, \hat{g}_\theta,M)$ given in (\ref{eqn:L_inter}).
\STATE Compute the intra-domain loss $L_{\mathcal{Y}\mathcal{Y}}(\hat{f}_\btheta, \hat{g}_\btheta, M)$ given in (\ref{eqn:LYY}).
\STATE Update $\hat{f}_\btheta$ and $\hat{g}_\btheta$ using the gradient $\nabla_\theta (L_\mathrm{sym}
   (\hat{f}_\btheta, \hat{g}_\btheta,M) + \lambda L_{\mathcal{Y}\mathcal{Y}}(\hat{f}_\btheta, \hat{g}_\btheta, M))$.
\ENDFOR
\ENSURE $\hat{f}_\btheta$, $\hat{g}_\btheta$
\end{algorithmic}
\label{alg:intra-regularization}
\end{algorithm}
\noindent Leveraging domain knowledge to improve optimization through intra-domain regularization, we introduce the Algorithm~\ref{alg:intra-regularization}. This approach introduces an additional hyperparameter $\lambda$ to manage the loss weights.

\subsection{\eande with acceptance-rejection sampling}
\label{sec:algorithm-sampling}
\begin{algorithm}[ht]
\caption{\eande with acceptance-rejection sampling algorithm \citep{eckhardt1987stan, bishop2006pattern}}
\renewcommand{\algorithmicrequire}{\textbf{Input:}}
\renewcommand{\algorithmicensure}{\textbf{Output:}}
\newcommand{\INDSTATE}[1][1]{\hspace{#1\algorithmicindent}}
\newcommand{\GreyComment}[1]{\textcolor{gray}{\footnotesize{$\triangleright$ \texttt{#1}}}}
\begin{algorithmic}[1]
\REQUIRE Observation $\boldsymbol{y}$,  prior distribution $p(\boldsymbol{\phi})$, proposal distribution $\pi(\bphi)$, sample size $S$, batch size $N'$, constant $B$, constant $\tau$.\\
\INDSTATE[-2] \textbf{Initialize:} Set $s = 1$.
\STATE Run the forward pass $\hat{f}_\btheta(\boldsymbol{y})$.
\STATE Sample $\{\bphi'\}_{i=1}^{N'} \sim p(\bphi)$ and compute $\hat{g}_\theta(\bphi_i)$ using a batch forward pass.
\STATE Estimate the normalization constant $C(\by)$ in a Monte-Carlo way using Equation \ref{eqn:normalization}.
\REPEAT 
\STATE Sample $\boldsymbol{\phi}^{(s)} \sim \pi(\boldsymbol{\phi})$ and $u\sim \mathcal{U}_{(0, 1)}$.
\STATE Run the forward pass $\hat{f}_\btheta(\by)$.
\STATE Calculate the likelihood-to-evidence ratio using $ \hat{r}_\btheta (\boldsymbol{\phi}^{(s)},\boldsymbol{y}) = C(\by)^{-1} e^{\hat{f}_\theta(\by)\cdot\hat{g}_\theta(\bphi^{(s)})/\tau}$.
\STATE Accept $\boldsymbol{\phi}^{(s)}$ if 
$u \leq \frac{
\hat{r}_\btheta(\boldsymbol{\phi}^{(s)},\boldsymbol{y})\,
p(\boldsymbol{\phi}^{(s)})}{B\pi(\boldsymbol{\phi}^{(s)})}$ 
and increase the counter $s$ by 1. Otherwise, reject.
\UNTIL{$s = S$}
\ENSURE $\{\boldsymbol{\phi}^{(0)}, \dots \boldsymbol{\phi}^{(S)}\}$.
\end{algorithmic}
\label{alg:acceptance-rejection-sampling}
\end{algorithm}
\noindent As demonstrated in Algorithm \ref{alg:acceptance-rejection-sampling}, using the acceptance-rejection sampling algorithm as an example, we explain in detail how to integrate \eande into the sampling process.

For a given observation $\by$, our approach requires only a single forward pass through the encoder $\hat{f}_\btheta$ to compute the embedding of $\by$. This embedding is then used to compute the likelihood-to-evidence ratio estimation across multiple iterations. 

In the standard acceptance-rejection sampling method \citep{eckhardt1987stan}, one might typically calculate the normalization constant, which we denote as $C(\by)$. As illustrated in steps 2 and 3 of Algorithm \ref{alg:acceptance-rejection-sampling}, we can estimate 
\begin{equation}
C(\by) \approx \sum_{i=1}^{N'} e^{\hat{f}_\btheta(\by)\cdot\hat{g}_\btheta(\bphi_i)/\tau}
\label{eqn:normalization}
\end{equation}
by drawing samples from the prior.
Estimating the normalization constant can improve computational efficiency but is not strictly necessary since the algorithm primarily relies on the ratio of the target density to the proposal density.

At each iteration of the algorithm, a candidate parameter $\bphi^{(s)}$ is sampled from the proposal distribution $\pi(\bphi)$, and the emulator $\hat{g}_\btheta(\bphi^{(s)})$ is used to compute the corresponding likelihood-to-evidence ratio estimation $\hat{r}(\bphi^{(s)}, \by)$. The acceptance decision is made based on comparing the scaled ratio estimation, which accounts for the prior density, to a uniformly sampled variable in the interval $[0,1)$.
This design eliminates the need for additional runs of the encoder, ensuring that the computationally expensive operation of encoding the high-dimensional data is performed only once. 
As a result, the efficiency of the sampling process is significantly enhanced, particularly when dealing with large datasets or complex models.

\section{Theoretical proofs and additional theoretical analysis}\label{apx:proofs}
This section presents comprehensive proofs of the statements given in Section~\ref{sec:theory} as well as additional results, including an analysis of the one-sided inter-domain InfoNCE loss (Appendix~\ref{sec:infonceproofs}, Corollary~\ref{cor:single}) and a discussion on using different priors during data generation and inference (Appendix~\ref{sec:altpriors}).

\subsection{Asymptotics of the InfoNCE loss}
\label{sec:asymptotics}
Following prior theoretical work on contrastive learning \citep{wang2020understanding,zimmermann2021contrastive}, our analysis will focus on the asymptotic case where the batch size $M\to\infty$. The inter-domain InfoNCE losses each decompose into two terms: the cross alignment and the negative cross-entropy of the latent embeddings $\hat{f}_\theta(\by),\hat{g}_\theta(\bphi)$.

\begin{lemma}\label{lem:asymtotics}
As the batch size $M\to\infty$, the inter-domain InfoNCE losses converge to 
\begin{align}
    \overline L_{\Phi\mathcal{Y}}(\hat{f}_\theta,\hat{g}_\theta) &:= \lim_{M\to\infty}L_{\Phi\mathcal{Y}}(\hat{f}_\theta,\hat{g}_\theta,M) - \log M\\
    &= -\frac{1}{\tau}\E_{(\bphi,\by) \sim p(\bphi,\by)}\left[\hat{f}_\theta(\by)\cdot \hat{g}_\theta(\bphi)\right] & \text{\em (cross alignment)}\\
    &\quad + \E_{\by \sim p(\by)}\left[\log\E_{\bphi^-\sim p(\bphi)}\left[e^{\hat{f}_\theta(\by)\cdot \hat{g}_\theta(\bphi^-)/\tau}\right]\right] & \text{\em (negative cross-entropy)}\\
    \overline L_{\mathcal{Y}\Phi}(\hat{f}_\theta,\hat{g}_\theta) &:= \lim_{M\to\infty}L_{\mathcal{Y}\Phi}(\hat{f}_\theta,\hat{g}_\theta,M) - \log M\\
    &= -\frac{1}{\tau}\E_{(\bphi,\by) \sim p(\bphi,\by)}\left[\hat{f}_\theta(\by)\cdot \hat{g}_\theta(\bphi)\right] & \text{\em (cross alignment)}\\
    &\quad + \E_{\bphi\sim p(\bphi)}\left[\log\E_{\by^- \sim p(\by)}\left[e^{\hat{f}_\theta(\by^-)\cdot \hat{g}_\theta(\bphi)/\tau}\right]\right]. & \text{\em (negative cross-entropy)}
\end{align}
\end{lemma}

\begin{proof}
We can write the inter-domain InfoNCE loss as
\begin{align}
    &L_{\Phi\mathcal{Y}}(\hat{f}_\btheta, \hat{g}_\btheta, M) - \log M\\
    &\qquad =  -\frac{1}{M}\sum_{i=1}^M\left[\frac{1}{\tau}\hat{f}_\btheta(\by_i)\cdot\hat{g}_\btheta(\bphi_i) - \log\left(\frac{1}{M}\sum_{j=1}^M e^{\hat{f}_\btheta(\by_i)\cdot\hat{g}_\btheta(\bphi_j)/\tau}\right)\right].
\end{align}
Then, taking the limit $M\to\infty$ by replacing
\begin{align}
    \frac{1}{M}\sum_{i=1}^M \longrightarrow \E_{(\bphi,\by) \sim p(\bphi,\by)} \text{\quad and \quad} \frac{1}{M}\sum_{j=1}^M \longrightarrow \E_{\bphi^- \sim p(\bphi)},
\end{align}
we find
\begin{align}
    \overline L_{\Phi\mathcal{Y}}(\hat{f}_\theta,\hat{g}_\theta) &:= \lim_{M\to\infty}L_{\Phi\mathcal{Y}}(\hat{f}_\theta,\hat{g}_\theta,M) - \log M\\
    &= -\frac{1}{\tau}\E_{(\bphi,\by) \sim p(\bphi,\by)}\left[\hat{f}_\theta(\by)\cdot \hat{g}_\theta(\bphi)\right] + \E_{\by \sim p(\by)}\left[\log\E_{\bphi^-\sim p(\bphi)}\left[e^{\hat{f}_\theta(\by)\cdot \hat{g}_\theta(\bphi^-)/\tau}\right]\right].
\end{align}
Similarly,
\begin{align}
    \overline L_{\mathcal{Y}\Phi}(\hat{f}_\theta,\hat{g}_\theta) &:= \lim_{M\to\infty}L_{\mathcal{Y}\Phi}(\hat{f}_\theta,\hat{g}_\theta,M) - \log M\\
    &= -\frac{1}{\tau}\E_{(\bphi,\by) \sim p(\bphi,\by)}\left[\hat{f}_\theta(\by)\cdot \hat{g}_\theta(\bphi)\right] + \E_{\bphi\sim p(\bphi)}\left[\log\E_{\by^- \sim p(\by)}\left[e^{\hat{f}_\theta(\by^-)\cdot \hat{g}_\theta(\bphi)/\tau}\right]\right].
\end{align}
This result is analogous to Theorem 1 in \citet{wang2020understanding}, which gives a similar result for the intra-domain InfoNCE loss.
\end{proof}

\subsection{Learning the parameter posterior by optimizing the InfoNCE loss}
\label{sec:infonceproofs}

\begin{definition}\label{def:modeldist}
Let $\hat{f}_\theta: \mathcal{Y} \to \mathbb{S}^{n-1}, \hat{g}_\theta: \Phi \to \mathbb{S}^{n-1}$ be learnable embedding functions, and define the model posterior and model likelihood distributions
\begin{align}
    \hat{q}_\theta(\bphi\mid\by) &:=  C_{\Phi\mathcal{Y}}(\by)^{-1}\,e^{\hat{f}_\theta(\by)\cdot \hat{g}_\theta(\bphi)/\tau}\,p(\bphi)\\
    \hat{q}_\theta(\by\mid\bphi) &:=  C_{\mathcal{Y}\Phi}(\bphi)^{-1}\,e^{\hat{f}_\theta(\by)\cdot \hat{g}_\theta(\bphi)/\tau}\,p(\by),
\end{align}
where
\begin{align}
    C_{\Phi\mathcal{Y}}(\by) &:= \int e^{\hat{f}_\theta(\by)\cdot \hat{g}_\theta(\bphi^-)/\tau}\,p(\bphi^-)\,\mathrm{d}\bphi^-\\
    C_{\mathcal{Y}\Phi}(\bphi) &:= \int e^{\hat{f}_\theta(\by^-)\cdot \hat{g}_\theta(\bphi)/\tau}\,p(\by^-)\,\mathrm{d}\by^-
\end{align}
are normalization factors. Note that, as defined, these two model distributions are not necessarily related $\hat{q}_\theta(\bphi\mid\by)\,p(\by) \ne \hat{q}_\theta(\by\mid\bphi)\,p(\bphi)$.
\end{definition}

\begin{lemma}\label{lem:crossent}
We can rewrite the asymptotic inter-domain InfoNCE losses as
\begin{align}
    \overline L_{\Phi\mathcal{Y}}(\hat{f}_\theta,\hat{g}_\theta) &= D_\mathrm{KL}(p(\bphi,\by)\,\|\,\hat{q}_\theta(\bphi\mid\by)\,p(\by)) - I(\bphi,\by)\\
    \overline L_{\mathcal{Y}\Phi}(\hat{f}_\theta,\hat{g}_\theta) &= D_\mathrm{KL}(p(\bphi,\by)\,\|\,\hat{q}_\theta(\by\mid\bphi)\,p(\bphi)) - I(\bphi,\by),
\end{align}
where $D_\mathrm{KL}$ is the Kullback--Leibler (KL) divergence and $I(\bphi,\by) := D_\mathrm{KL}(p(\bphi,\by)\,\|\,p(\bphi)\,p(\by))$ is the mutual information between $\bphi$ and $\by$. Note that $I(\bphi,\by)$ is a constant, so minimizing these losses is equivalent to minimizing a KL divergence.
\end{lemma}

\begin{proof}
Noting that the normalization constant can be written as
\begin{align}
    C_{\Phi\mathcal{Y}}(\by) = \E_{\bphi^-\sim p(\bphi)}\left[e^{\hat{f}_\theta(\by)\cdot \hat{g}_\theta(\bphi^-)/\tau}\right],
\end{align}
the asymptotic inter-domain InfoNCE loss (Lemma~\ref{lem:asymtotics})
\begin{align}
    \overline L_{\Phi\mathcal{Y}}(\hat{f}_\theta,\hat{g}_\theta) &= -\frac{1}{\tau}\E_{(\bphi,\by) \sim p(\bphi,\by)}\left[\hat{f}_\theta(\by)\cdot \hat{g}_\theta(\bphi)\right] + \E_{\by \sim p(\by)}\left[\log\E_{\bphi^-\sim p(\bphi)}\left[e^{\hat{f}_\theta(\by)\cdot \hat{g}_\theta(\bphi^-)/\tau}\right]\right]\\
    &= -\frac{1}{\tau}\E_{(\bphi,\by) \sim p(\bphi,\by)}\left[\hat{f}_\theta(\by)\cdot \hat{g}_\theta(\bphi)\right] + \E_{\by \sim p(\by)}\left[\log C_{\Phi\mathcal{Y}}(\by)\right]\\
    &=\E_{(\bphi,\by) \sim p(\bphi,\by)}\left[-\log\left(C_{\Phi\mathcal{Y}}(\by)^{-1}\,e^{\hat{f}_\theta(\by)\cdot \hat{g}_\theta(\bphi)/\tau}\,p(\bphi)\right) + \log p(\bphi)\right]\\
    &=\E_{(\bphi,\by) \sim p(\bphi,\by)}\left[-\log\hat{q}_\theta(\bphi\mid\by) + \log p(\bphi)\right]\\
    &=\E_{(\bphi,\by) \sim p(\bphi,\by)}\left[-\log\frac{\hat{q}_\theta(\bphi\mid\by)\,p(\by)}{p(\bphi,\by)} - \log\frac{p(\bphi,\by)}{p(\bphi)\,p(\by)}\right]\\
    &= D_\mathrm{KL}(p(\bphi,\by)\,\|\,\hat{q}_\theta(\bphi\mid\by)\,p(\by)) - I(\bphi,\by).
\end{align}
Similarly, by swapping $\hat{f}\leftrightarrow\hat{g}$, $\by\leftrightarrow\bphi$, and $\mathcal{Y}\leftrightarrow\Phi$, we can derive
\begin{align}
    \overline L_{\mathcal{Y}\Phi}(\hat{f}_\theta,\hat{g}_\theta) &= D_\mathrm{KL}(p(\bphi,\by)\,\|\,\hat{q}_\theta(\by\mid\bphi)\,p(\bphi)) - I(\bphi,\by).
\end{align}
\end{proof}

\assumpflexiblemain*

\begin{lemma}\label{lem:ratioconstant}
Given Assumption~\ref{assump:flexiblemain}, the likelihood-to-evidence ratio can be written
\begin{align}
    r(\bphi,\by) := \frac{p(\bphi,\by)}{p(\bphi)\,p(\by)} 
    &= \hat{q}_{\theta^*}(\bphi\mid\by)/p(\bphi)\\
    &= \hat{q}_{\theta^*}(\by\mid\bphi)/p(\by)\\
    &= C^{*-1}\,e^{\hat{f}_{\theta^*}(\by)\cdot \hat{g}_{\theta^*}(\bphi)/\tau},
\end{align}
where the normalization constant
\begin{align}
    C^* := C_{\Phi\mathcal{Y}}(\by; \theta^*) = C_{\mathcal{Y}\Phi}(\bphi; \theta^*)
\end{align}
does not vary with $\by$ or $\bphi$.
\end{lemma}

\begin{proof}
Using Assumption~\ref{assump:flexiblemain} and recalling Definition~\ref{def:modeldist},
\begin{align}
    \hat{q}_{\theta^*}(\bphi\mid\by)/p(\bphi) &= \hat{q}_{\theta^*}(\by\mid\bphi)/p(\by)\\
    C_{\Phi\mathcal{Y}}(\by; \theta^*)^{-1}\,e^{\hat{f}_{\theta^*}(\by)\cdot \hat{g}_{\theta^*}(\bphi)/\tau} &= C_{\mathcal{Y}\Phi}(\bphi; \theta^*)^{-1}\,e^{\hat{f}_{\theta^*}(\by)\cdot \hat{g}_{\theta^*}(\bphi)/\tau}\\
    C_{\Phi\mathcal{Y}}(\by; \theta^*) &= C_{\mathcal{Y}\Phi}(\bphi; \theta^*)
\end{align}
for all $\by \in \mathcal{Y}$ and $\bphi \in \Phi$, we can conclude that $C_{\Phi\mathcal{Y}}(\by; \theta^*)$ and $C_{\mathcal{Y}\Phi}(\bphi; \theta^*)$ must independent of both $\by$ and $\bphi$.
\end{proof}

\thmmain*

\begin{proof}
Lemma~\ref{lem:crossent} implies that
\begin{align}
    \overline L_\mathrm{sym}(\hat{f}_\theta,\hat{g}_\theta) &= \overline L_{\Phi\mathcal{Y}}(\hat{f}_\theta,\hat{g}_\theta) + \overline L_{\mathcal{Y}\Phi}(\hat{f}_\theta,\hat{g}_\theta)\\
    \begin{split}
    &= D_\mathrm{KL}(p(\bphi,\by)\,\|\,\hat{q}_\theta(\bphi\mid\by)\,p(\by))\\
    &\qquad + D_\mathrm{KL}(p(\bphi,\by)\,\|\,\hat{q}_\theta(\by\mid\bphi)\,p(\bphi)) - 2\,I(\bphi,\by)
    \end{split}\label{eqn:symlossKL}\\
    &\ge D_\mathrm{KL}(p(\bphi,\by)\,\|\,\hat{q}_\theta(\bphi\mid\by)\,p(\by)) - 2\,I(\bphi,\by)\\
    &\ge - 2\,I(\bphi,\by).
\end{align}
Given Assumption~\ref{assump:flexiblemain}, there exists $\theta^*$ such that
\begin{align}
    p(\bphi\mid\by) &= \hat{q}_{\theta^*}(\bphi\mid\by)\\
    p(\by\mid\bphi) &= \hat{q}_{\theta^*}(\by\mid\bphi).
\end{align}
Plugging this into (\ref{eqn:symlossKL}), we find
\begin{align}
    \overline L_\mathrm{sym}(\hat{f}_{\theta^*},\hat{g}_{\theta^*}) = - 2\,I(\bphi,\by) = \min_\theta \overline L_\mathrm{sym}(\hat{f}_\theta,\hat{g}_\theta),
\end{align}
i.e., $\theta^*$ is a global minimizer for $\overline L_\mathrm{sym}$. Furthermore, since $D_\mathrm{KL}(a\,\|\,b) = 0 \Leftrightarrow a = b$, any global minimizer $\theta^* \in \argmin_\theta \overline L_\mathrm{sym}(\hat{f}_\theta,\hat{g}_\theta)$ must have
\begin{align}
    \hat{q}_{\theta^*}(\bphi\mid\by) &= p(\bphi\mid\by)\\
    \hat{q}_{\theta^*}(\by\mid\bphi) &= p(\by\mid\bphi),
\end{align}
and therefore, by Lemma~\ref{lem:ratioconstant},
\begin{align}
    \hat{r}_{\theta^*}(\bphi,\by) &:= \hat{q}_{\theta^*}(\bphi\mid\by)/p(\bphi) = C^{*-1}\,e^{\hat{f}_{\theta^*}(\by)\cdot \hat{g}_{\theta^*}(\bphi)/\tau} = r(\bphi,\by).
\end{align}
\end{proof}

\begin{corollary}\label{cor:single}
The asymptotic one-sided inter-domain InfoNCE loss has the form
\begin{align}
    \overline L_{\Phi\mathcal{Y}}(\hat{f}_\theta,\hat{g}_\theta) &= D_\mathrm{KL}(p(\bphi,\by)\,\|\,\hat{q}_\theta(\bphi\mid\by)\,p(\by)) - I(\bphi,\by).
\end{align}
Assuming there exists $\theta^*$ such that $p(\bphi\mid\by) = \hat{q}_{\theta^*}(\bphi\mid\by)$ ((\ref{eqn:assump1}) from Assumption~\ref{assump:flexiblemain}), the global minimum is
\begin{align}
    \min_\theta \overline L_{\Phi\mathcal{Y}}(\hat{f}_\theta,\hat{g}_\theta) = -I(\bphi,\by),
\end{align}
and, for any global minimizer $\theta^* \in \argmin_{\theta}  \overline L_{\Phi\mathcal{Y}}(\hat{f}_\theta,\hat{g}_\theta)$, the model posterior
\begin{align}
    \hat{q}_{\theta^*}(\bphi\mid\by) := \hat{r}_{\theta^*}(\bphi,\by)\,p(\bphi) = p(\bphi\mid\by)
\end{align}
and model likelihood-to-evidence ratio
\begin{align}
    \hat{r}_{\theta^*}(\bphi,\by) = C(\by)^{-1}\,e^{\hat{f}_{\theta^*}(\by)\cdot \hat{g}_{\theta^*}(\bphi)/\tau} = r(\bphi,\by),
\end{align}
where $C(\by)$ is data-dependent normalization factor.
\end{corollary}

\begin{proof}
This follows analogously to the proof of Theorem~\ref{thm:main}. Lemma~\ref{lem:crossent} implies that
\begin{align}
    \overline L_{\Phi\mathcal{Y}}(\hat{f}_\theta,\hat{g}_\theta) &= D_\mathrm{KL}(p(\bphi,\by)\,\|\,\hat{q}_\theta(\bphi\mid\by)\,p(\by)) - I(\bphi,\by) \ge -I(\bphi,\by).
\end{align}
Thus, given there exists $\theta^*$ such that $p(\bphi\mid\by) = \hat{q}_{\theta^*}(\bphi\mid\by)$, $\theta^*$ must be a global minimizer with
\begin{align}
    \overline L_{\Phi\mathcal{Y}}(\hat{f}_{\theta^*},\hat{g}_{\theta^*}) = -I(\bphi,\by) = \min_\theta \overline L_{\Phi\mathcal{Y}}(\hat{f}_\theta,\hat{g}_\theta).
\end{align}
Furthermore, for any $\theta^* \in \argmin_\theta \overline L_{\Phi\mathcal{Y}}(\hat{f}_\theta,\hat{g}_\theta)$,
\begin{align}
    \hat{q}_{\theta^*}(\bphi\mid\by) = p(\bphi\mid\by)
\end{align}
and therefore
\begin{align}
    \hat{r}_{\theta^*}(\bphi,\by) &= \hat{q}_{\theta^*}(\bphi\mid\by)/p(\bphi) = C(\by)^{-1}\,e^{\hat{f}_{\theta^*}(\by)\cdot \hat{g}_{\theta^*}(\bphi)/\tau} = r(\bphi,\by).
\end{align}
\end{proof}

\corsufficientstatisticmain*

\begin{proof}
Given Theorem~\ref{thm:main} or Corollary~\ref{cor:single}, we can write the posterior as
\begin{align}
    p(\bphi\mid\by) = \hat{q}_{\theta^*}(\bphi\mid\by) =  C'(\hat{f}_{\theta^*}(\by))^{-1}\,e^{\hat{f}_{\theta^*}(\by)\cdot \hat{g}_{\theta^*}(\bphi)/\tau}\,p(\bphi),
\end{align}
where we have rewritten the normalization factor $C(\by)$ as $C'(\hat{f}_{\theta^*}(\by))$ to emphasize its dependence on only $\hat{f}_{\theta^*}(\by)$. This fulfills the condition for Bayesian sufficiency
\begin{align}
    p(\bphi\mid\by) = p(\bphi\mid\hat{f}_{\theta^*}(\by)).
\end{align}

We can also write the likelihood as
\begin{align}
    p(\by\mid\bphi) = p(\bphi\mid\by)\,p(\by)/p(\bphi) &=  C'(\hat{f}_{\theta^*}(\by))^{-1}\,e^{\hat{f}_{\theta^*}(\by)\cdot \hat{g}_{\theta^*}(\bphi)/\tau}\,p(\by)\\
    &= a(\by)\,b_\bphi(\hat{f}_{\theta^*}(\by)),
\end{align}
where $a(\by) = p(\by)$ and $b_\bphi(\hat{f}_{\theta^*}(\by)) = C'(\hat{f}_{\theta^*}(\by))^{-1}\,e^{\hat{f}_{\theta^*}(\by)\cdot \hat{g}_{\theta^*}(\bphi)/\tau}$. Therefore, $\hat{f}_{\theta^*}(\by)$ is a sufficient statistic for $\bphi$ by the Fisher--Neyman factorization theorem.
\end{proof}

\subsection{Optimal data compression and non-identifiable parameters}

\thmembeddings*

\begin{proof}
Since
\begin{align}
    r(\bphi,\by) &= \frac{p(\bphi\mid\by)}{p(\bphi)} = \frac{p(\bphi\mid S(\by))}{p(\bphi)}\\
    r(\bphi,\by) &= \frac{p(\by\mid\bphi)}{p(\by)} = \frac{p(\by\mid\Pi(\bphi))}{p(\by)},
\end{align}
the likelihood-to-evidence ratio
\begin{align}
    r(\bphi,\by) = r'(\Pi(\bphi), S(\by))
\end{align}
must be expressible as a function $r'$ of only $S(\by)$ and $\Pi(\bphi)$. From Theorem~\ref{thm:main}, we have
\begin{align}
    \hat{f}_{\theta^*}(\by)\cdot \hat{g}_{\theta^*}(\bphi) &= \tau(\log r(\bphi,\by) + \log C^*)\\
    &= \tau(\log r'(\Pi(\bphi), S(\by)) + \log C^*),
\end{align}
so $(\by,\bphi) \mapsto \hat{f}_{\theta^*}(\by)\cdot\hat{g}_{\theta^*}(\bphi)$ must also be expressible as a function of only $S(\by)$ and $\Pi(\bphi)$. This means that $\forall \by_1,\by_2 \in \mathcal{Y} : S(\by_1) = S(\by_2)$ and $\forall \bphi_1,\bphi_2 \in \Phi : \Pi(\bphi_1) = \Pi(\bphi_2)$,
\begin{align}\label{eqn:embeddingequiv}
    \hat{f}_{\theta^*}(\by_1)\cdot \hat{g}_{\theta^*}(\bphi_1) = \hat{f}_{\theta^*}(\by_2)\cdot \hat{g}_{\theta^*}(\bphi_2),
\end{align}
i.e., changes in $\by,\bphi$ that do not alter $S(\by),\Pi(\bphi)$ also do not affect the likelihood-to-evidence ratio.\\

\noindent Since $S$ and $\Pi$ are surjective, let $S^\dagger: \mathcal{M}\to\mathcal{Y}$ be a right-inverse of $S$ (i.e., $S \circ S^\dagger = \mathrm{id}_\mathcal{M}$), and let $\Pi^\dagger: \Psi\to\Phi$ be a right-inverse of $\Pi$ (i.e., $\Pi \circ \Pi^\dagger = \mathrm{id}_\Psi$).\footnote{Note that surjectivity is not a true constraint on $S$ and $\Pi$ since we can always redefine $\mathcal{M} = \mathrm{Im}(S)$ and $\Psi = \mathrm{Im}(\Pi)$ so that $S$ and $\Pi$ are always surjective.} Now, define $\hat{f}_\mathcal{M}:\mathcal{M}\to\mathbb{S}^{n-1}$ and $\hat{g}_\Psi:\Psi\to\mathbb{S}^{n-1}$ to be
\begin{align}
    \hat{f}_\mathcal{M} &:= \hat{f}_{\theta^*}\circ S^\dagger\\
    \hat{g}_\Psi &:= \hat{g}_{\theta^*}\circ \Pi^\dagger.
\end{align}
Then, because $S(S^\dagger\circ S(\by)) = S(\by)$ and $\Pi(\Pi^\dagger\circ\Pi(\bphi)) = \Pi(\bphi)$, we can use (\ref{eqn:embeddingequiv}) to conclude that, for all $\by \in \mathcal{Y},\bphi \in \Phi$,
\begin{align}
    \hat{f}_{\theta^*}(\by)\cdot \hat{g}_{\theta^*}(\bphi) &= \hat{f}_{\theta^*}(S^\dagger \circ S(\by))\cdot \hat{g}_{\theta^*}(\Pi^\dagger \circ \Pi(\bphi))\\
    &= \hat{f}_\mathcal{M}(S(\by))\cdot\hat{g}_\Psi(\Pi(\bphi)),
\end{align}
and thus
\begin{align}
    r(\bphi,\by) = r'(\Pi(\bphi), S(\by)) = C^{*-1}\,e^{\hat{f}_\mathcal{M}(S(\by))\cdot \hat{g}_\Psi(\Pi(\bphi))/\tau}.
\end{align}
\end{proof}

\corminimal*

\begin{proof}
Let $S(\by)$ be a minimal sufficient statistic for $\bphi$ (with $\by \sim G(\bphi)$), and $\Pi(\bphi) = \bphi$ be the identity. Applying Theorem~\ref{thm:embeddings} and letting
\begin{align}
    \hat{f}_\mathrm{opt} := \hat{f}_\mathcal{M}\circ S,
\end{align}
we can express the likelihood-to-evidence ratio as
\begin{align}
    r(\bphi,\by) = C^{*-1}\,e^{\hat{f}_\mathrm{opt}(\by)\cdot \hat{g}_{\theta^*}(\bphi)/\tau}.
\end{align}
Furthermore, due to the exponential family form of the likelihood-to-evidence ratio, for $\by \sim G(\bphi)$, $\hat{f}_\mathrm{opt}(\by)$ is a sufficient statistic for $\bphi$. In fact, $\hat{f}_\mathrm{opt}(\by)$ is a \emph{minimal} sufficient statistic for $\bphi$ since $\hat{f}_\mathrm{opt}(\by) := \hat{f}_\mathcal{M}(S(\by))$ is a function of another minimal sufficient statistic $S(\by)$.
\end{proof}

\corredundant*

\begin{proof}
Let $S(\by) = \by$ be the identity and $\Pi(\bphi) = \boldsymbol{\psi} : \bphi \in \bphi \in \Phi_\mathrm{eq}^{(\boldsymbol{\psi})}$ be the projection operator. Applying Theorem~\ref{thm:embeddings} and letting
\begin{align}
    \hat{g}_\mathrm{id} := \hat{g}_\Psi\circ \Pi,
\end{align}
we can express the likelihood-to-evidence ratio as
\begin{align}
    r(\bphi,\by) = C^{*-1}\,e^{\hat{f}_{\theta^*}(\by)\cdot \hat{g}_\mathrm{id}(\bphi)/\tau}.
\end{align}
\end{proof}

\subsection{Understanding the learned latent space}

\lemnormconstant*

\begin{proof}
Using the constraints from Definition~\ref{def:toymodel}, we can write the marginal
\begin{align}
    p(\by) &= \int_\mathcal{Z} p(\by\mid\bz)\,p(\bz)\,\mathrm{d}\bz\\
    &= \int_\mathcal{Z} C_{\mathcal{Y}\mathcal{Z}}^{-1}\,\eta(\by)\,\delta(f(\by) - \bz)\,|\mathcal{Z}|^{-1}\,\mathrm{d}\bz\\
    &= C_{\mathcal{Y}\mathcal{Z}}^{-1}|\mathcal{Z}|^{-1}\eta(\by)\\
    &= \eta(\by),
\end{align}
where $C_{\mathcal{Y}\mathcal{Z}} = |\mathcal{Z}|^{-1}$ since both $p(\by)$ and $\eta(\by)$ are normalized distributions. Inserting this result into (\ref{eqn:toymodel}) and integrating, we obtain
\begin{align}
    p(\by\mid \bphi) &= |\mathcal{Z}|C_\kappa^{-1}\left[\int_\mathcal{Z} e^{\kappa\,\bz\cdot g(\bphi)}\,\delta(f(\by) - \bz)\,\mathrm{d}\bz\right]\,p(\by)\\
    &= |\mathcal{Z}|C_\kappa^{-1}e^{\kappa\,f(\by)\cdot g(\bphi)}\,p(\by).
\end{align}
\end{proof}

\thmrotation*

\begin{proof}
Using the result from Theorem~\ref{thm:main} and Lemma~\ref{lem:normconstant}, we have
\begin{align}
    \hat{q}_{\theta^*}(\bphi,\by) &= p(\bphi,\by)\\
    C^{*-1}\,e^{\hat{f}_{\theta^*}(\by)\cdot \hat{g}_{\theta^*}(\bphi)/\tau}\,p(\by)\,p(\bphi) &= |\mathcal{Z}|C_{\kappa}^{-1}e^{\kappa\,f(\by)\cdot g(\bphi)}\,p(\by)\,p(\bphi)
\end{align}
so, given $\tau = 1/\kappa$,
\begin{align}\label{eqn:matchingcondition}
    \hat{f}_{\theta^*}(\by)\cdot \hat{g}_{\theta^*}(\bphi) &= f(\by)\cdot g(\bphi) +\log(|\mathcal{Z}|C^*/C_{\kappa})/\kappa.
\end{align}
Recalling that $\hat{f}_{\theta^*}(\by), \hat{g}_{\theta^*}(\bphi), f(\by), g(\bphi) \in \mathbb{S}^{n-1}$, we have $\hat{f}_{\theta^*}(\by)\cdot \hat{g}_{\theta^*}(\bphi),\ f(\by)\cdot g(\bphi) \in [-1,1]$ where $f(\by)\cdot g(\bphi)$ is known to cover the full range due to constraint (iii) in Definition~\ref{def:toymodel}. Since (\ref{eqn:matchingcondition}) must hold for all $\by$, $\bphi$, the additional constant $\log(|\mathcal{Z}|C^*/C_{\kappa})/\kappa$ must vanish, giving
\begin{align}\label{eqn:innerproductpreserved}
    \hat{f}_{\theta^*}(\by)\cdot \hat{g}_{\theta^*}(\bphi) &= f(\by)\cdot g(\bphi).
\end{align}
By an extension of the Mazur--Ulam theorem \citep{zimmermann2021contrastive}, since (\ref{eqn:innerproductpreserved}) implies the metric on $\mathbb{S}^{n-1}$ is preserved, the learned embedding and the original latent space can only differ by an isometry, i.e.,
\begin{align}
    \hat{f}_{\theta^*}(\by) &= Rf(\by)\\
    \hat{g}_{\theta^*}(\bphi) &= Rg(\bphi)
\end{align}
for some orthogonal matrix $R \in \mathrm{SO}(n)$.
\end{proof}

\subsection{Using alternative parameter inference priors}\label{sec:altpriors}
In many applications, the parameter prior distribution that we want to use during inference is different from the prior used to collect or generate the data. In fact, the true parameter prior in the data is often unknown and must be empirically estimated if required. Fortunately, our approach allows us to easily use a different prior $\widetilde{p}(\bphi)$ during inference than the prior $p(\bphi)$ from the data.

\begin{corollary}
Using an alternative prior $\widetilde{p}(\bphi)$ gives the posterior
\begin{align}
    \widetilde{p}(\bphi\mid\by) \propto p(\by\mid\bphi)\,\widetilde{p}(\bphi).
\end{align}
Then, from Theorem~\ref{thm:main} and assuming $\mathrm{supp}(\widetilde{p}(\bphi)) \subseteq \mathrm{supp}(p(\bphi))$, we have
\begin{align}
    \widetilde{q}_{\theta^*}(\bphi\mid\by) &:= \widetilde{C}^{*-1}(\by)\,e^{\hat{f}_{\theta^*}(\by)\cdot \hat{g}_{\theta^*}(\bphi)/\tau}\,\widetilde{p}(\bphi)\\
    &= \frac{C^*}{\widetilde{C}^*(\by)}\cdot \frac{p(\bphi\mid\by)\,\widetilde{p}(\bphi)}{p(\bphi)}\\
    &= \frac{C^*}{\widetilde{C}^*(\by)}\cdot \frac{p(\by\mid\bphi)\,\widetilde{p}(\bphi)}{p(\by)}\\
    &= \widetilde{p}(\bphi\mid\by),
\end{align}
where the normalization constant $\widetilde{C}^*(\by) = C^*\,\widetilde{p}(\by)/p(\by)$.
\end{corollary}

By using a different inference prior, we lose the nice data-independent property of the normalizing constant. However, if the inference prior is sufficiently similar to the data prior such that the ratio of the data marginals $\widetilde{p}(\by)/p(\by)$ is slowly varying, we still have a slowly varying normalization constant which likely retains much of the performance benefits of the data independent constant.

\section{Experimental details}
In this section, we provide a detailed description of the experimental setup for the experiments presented in Section~\ref{sec:experiments}, including evaluation metrics, data generation, and hyperparameter selection.

\subsection{Synthetic experiments}
\label{apx:exp-synthetic}
\textbf{Data generation.} To set up the invertible MLP $f$, we adapt the setup from \citet{hyvarinen2016unsupervised, zimmermann2021contrastive}.
Specifically, we use five hidden layers with leaky ReLU units and randomly initialized weights and ensure the invertibility of $f$ by controlling the conditional number of the weight matrices.
We set the transformation matrix in the parameter space as 
$A = \begin{bmatrix} 
0.5 & 0.2 \\ 0.0 & 0.8
\end{bmatrix}$.

\vspace{5pt}
\noindent \textbf{Validation metrics.}
For the validation purpose, given a set of parameters-data pairs $\{(\bphi_i, \by_i)\}_{i=1}^{N_{\rm val}} \overset{i.i.d.}{\sim} p(\bphi, \by)$, we use the values of the estimated posterior as validation metric:
\begin{equation} \label{metric:validation}
    \mathrm{Acc.}(\hat{q}) = \sum_{i=1}^{N_{\rm val}} \hat{q}(\bphi_i \mid \by_i).
\end{equation}
Here, the goal is to maximize the posterior for matched data-parameters pairs sampled from the joint distribution.
For the synthetic toy experiment, we set the size of the validation dataset as $N_{\rm val} = 100$, and tune the model based on the median value of this metric across the entire validation dataset.

\vspace{5pt}
\noindent \textbf{Evaluation metrics.}
For each observation $\by$, we report the $l1$ distance between the estimated posterior and the true posterior following:
\begin{equation} \label{metric:l1-distance}
\mathrm{Dist.}(\hat{q}, p) = \sum_{i=1}^{N_{\rm eval}} |\hat{q}(\bphi_i \mid \by) - p(\bphi_i \mid \by) |,
\end{equation}
where a set of size $N_{\rm eval}=10000$ parameters are sampled from the prior distribution $p(\bphi)$.

\vspace{5pt}
\noindent \textbf{Implementation details.} 
We configure both the encoder and the emulator using the residual-connected architecture described in \citet{jiang2022embed}, where the width of hidden layers is chosen using grid search from a set of $\{60, 90, 120, 150\}$.
We set the total training epochs as $2000$ and use the cosine learning rate scheduler during training.
We choose the initial learning rate using the grid search from $\{ 5e^{-4}, 1e^{-3}\}$.

Since our objective in Section \ref{sec:vmf} is to confirm the connection between the embedding and the latent space defined in the data generative process, we set the embedding dimension to $2$, matching the dimensionality of the latent space, for both the unimodal and multimodal scenarios.
We set $\tau = \frac{1}{\kappa}$ following Theorem~\ref{thm:rotation}.

\subsection{High-dimensional Lorenz 96 experiments}
\label{apx:l96}
\textbf{Validation metrics.} As in Section~\ref{apx:exp-synthetic}, we use the estimated posterior values as the validation metrics. We set the size of the validation data set as $N_{\rm val}=50$.

\vspace{5pt}
\noindent\textbf{Evaluation metrics.} 
We use the acceptance-rejection sampling for drawing the samples from the modeled posterior, following the implementation in \citet{tejero-cantero2020sbi}.
We then evaluate the quality of the estimated posterior $\hat{q}(\bphi\mid\by^o)$ using the sample-based maximum mean discrepancy (MMD) metric. The MMD between two collections of samples drawn from two distributions is defined as:
\begin{align*}
    \widehat{\text{MMD}}^2 & \big(\{ \hat{\bphi}_{j} \}_{j=1}^M, \{{\bphi}_{j} \}_{j=1}^M \big) =  \\
    & \left( \frac{1}{M^2} \sum_{j=1}^M \sum_{j'=1}^M k(\hat{\bphi}_j, \hat{\bphi}_{j'}) \right) - \left( \frac{2}{M^2} \sum_{j=1}^M \sum_{j'=1}^M k(\hat{\bphi}_i, \bphi_{j'}) \right) + \left( \frac{1}{M^2} \sum_{j=1}^M \sum_{j'=1}^M k({\bphi}_j, {\bphi}_{j'}) \right),
\end{align*}
where $k(\cdot, \cdot)$ corresponds to a Gaussian kernel with standard deviation $\sigma$. A smaller $\sigma$ typically allows for a more precise detection of fine-grained differences between the distributions. 

\vspace{5pt}
\noindent \textbf{Implementation details.}
We train \eande for 2000 epochs utilizing a cosine learning rate scheduler starting at $1e^{-3}$. For {\sc NPE-C} \citep{greenberg2019automatic} and {\sc NRE-C} \citep{miller2022contrastive}, we perform training for a total of 4000 epochs to guarantee convergence, using a learning rate of $1e^{-4}$ as we observed improved performance with a reduced learning rate. 
For all methods, we evaluate earlier checkpoints using the validation metrics and select the one with the highest validation accuracy to evaluate on the testing data.
For all experiments, we use the AdamW optimizer with weight decay $5e^{-4}$.
For the size of $N=500$ training dataset, we use batch size $M=500$ for \eande and {\sc NPE-C}.
In \eande, to enhance the convergence of the InfoNCE loss, we employ a memory bank method \citep{he2020momentum} to empirically increase the number of negative samples $M$.
Specifically, we set $M$ equal to the training batch size by storing representations from previous mini-batches.

All hyperparameters are chosen using 
the reserved validation set. The range of values searched over are as follows:
\begin{itemize}
\item For \eande, the temperature values $\tau$ controlling the radius of the hypersphere were selected from the set $\{1e^{-4}, 1e^{-3}, 1e^{-2}, 1e^{-1} \}$.
\item For \eande, when we choose to use the intra-domain InfoNCE loss, we selected its weight $\lambda$ from the set $\{0, 0.2, 0.4, 0.6, 0.8, 1.0\}$.
\item For {\sc NRE-C}, we choose the hyperparameter that implies the odds that the pairs are drawn dependently to completely independently ($\gamma$ in \citet{miller2022contrastive}) from the set $\{1e^{-4}, 1e^{-3}, 1e^{-2}, 1e^{-1}\}$.
\item For {\sc NRE-C}, we try to increase the hyperparameter value controlling the number of classes ($K$ in \citet{miller2022contrastive}). However, as the number of classes increases, the number of required forward passes scales linearly. Given the memory constraints and a fixed batch size, it is not feasible to increase the number of classes indefinitely. Therefore, in four parallel GPU training sessions, we search for the optimal batch size from the set $\{40, 60, 80, 100, 120\}$ and the largest number of classes that are allowed from the set $\{4, 8, 12, 16\}$ per GPU.
\item For all methods, we choose the embedding dimensionality from the set $\{128, 256, 512\}$.
\end{itemize}

\noindent \textbf{Computational resources.} Training of \eande and {\sc NRE-C} was performed on a system with 4x Nvidia A40 GPUs, 2 AMD EPYC 7302 CPUs, and 128GB of RAM. Training of {\sc NPE-C} was performed on a system with 1x Nvidia A40 GPUs, 2 AMD EPYC 7302 CPUs, and 128GB of RAM. Evaluation for all three methods was performed on a system with 1x Nvidia A40 GPUs, 2 AMD EPYC 7302 CPUs, and 128GB of RAM.

\subsection{Additional visualizations}
\label{apx:visualizations}
\begin{figure}[t]
\centering
\begin{minipage}{1\textwidth}
  \centering
  \begin{subfigure}{.3\textwidth}
    \centering
    \includegraphics[trim={0.3cm 0.5cm 7cm 6.6cm},clip,height=4.5cm]{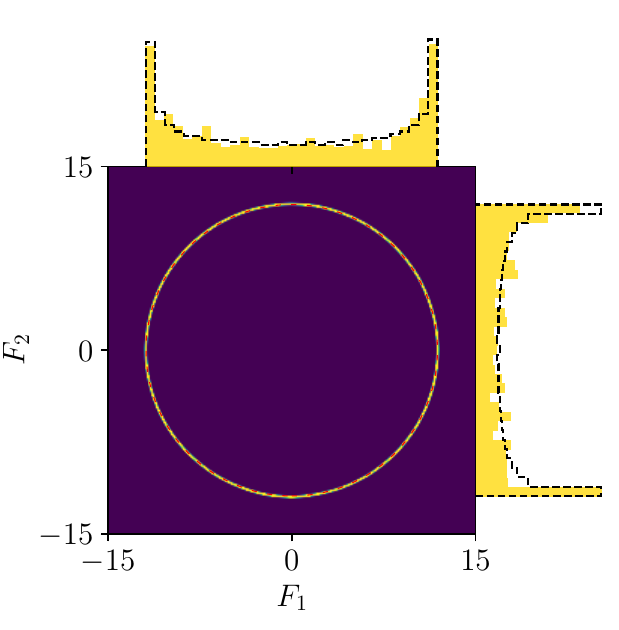}
    \caption{\eande}
  \end{subfigure}%
  \hfill
  \begin{subfigure}{.3\textwidth}
    \centering
    \includegraphics[trim={0.3cm 0.5cm 7cm 6.6cm},clip,height=4.5cm]{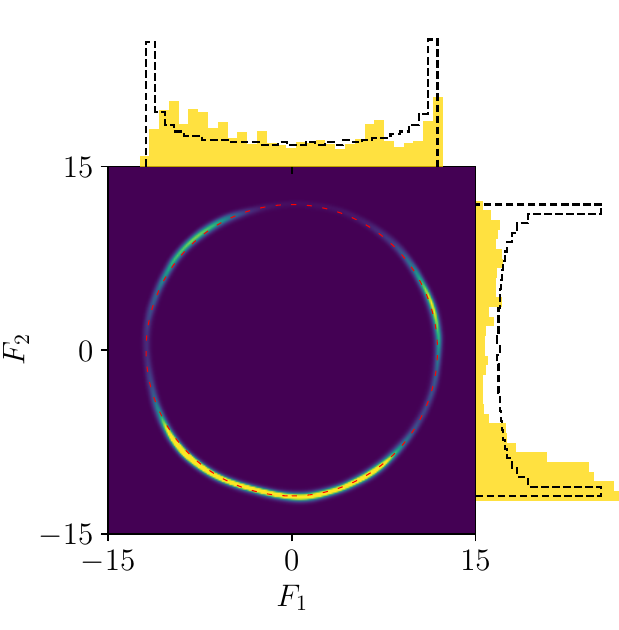}
    \caption{NRE-C}
  \end{subfigure}%
  \hfill
  \begin{subfigure}{.37\textwidth}
    \centering
    \includegraphics[trim={0.3cm 0.5cm 3cm 6.6cm},clip, height=4.5cm]{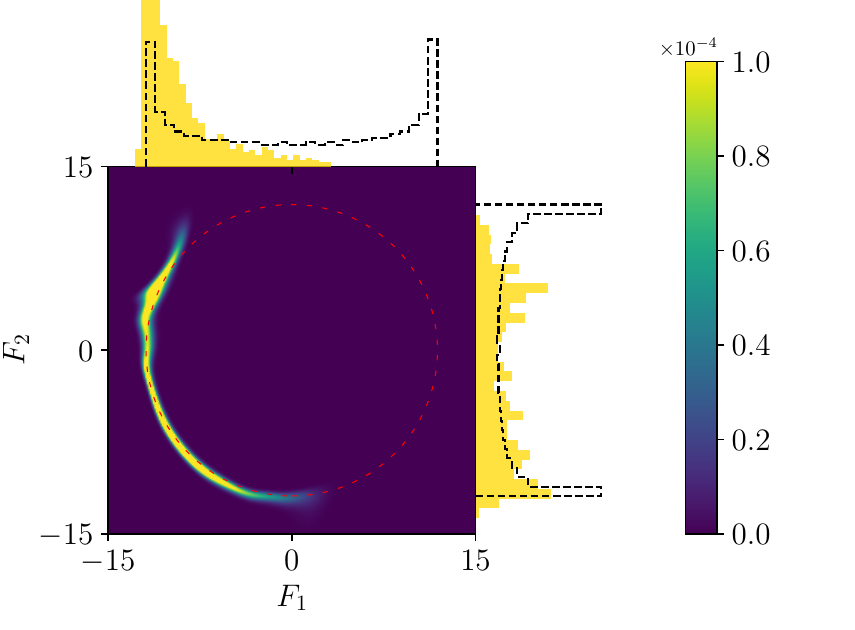}
    \caption{NPE-C}
  \end{subfigure}
\end{minipage}

\begin{minipage}{1\textwidth}
  \centering
  \begin{subfigure}{.3\textwidth}
    \centering
    \includegraphics[trim={0.3cm 0.5cm 7cm 6.6cm},clip,height=4.5cm]{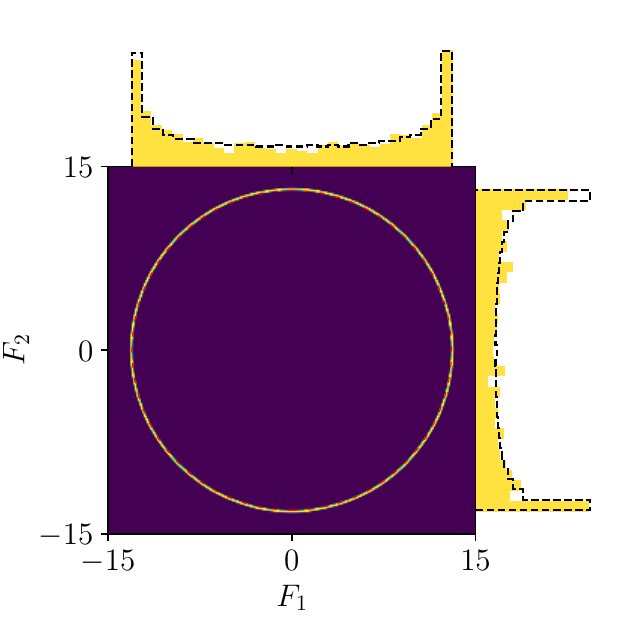}
    \caption{\eande}
  \end{subfigure}%
  \hfill
  \begin{subfigure}{.3\textwidth}
    \centering
    \includegraphics[trim={0.3cm 0.5cm 7cm 6.6cm},clip,height=4.5cm]{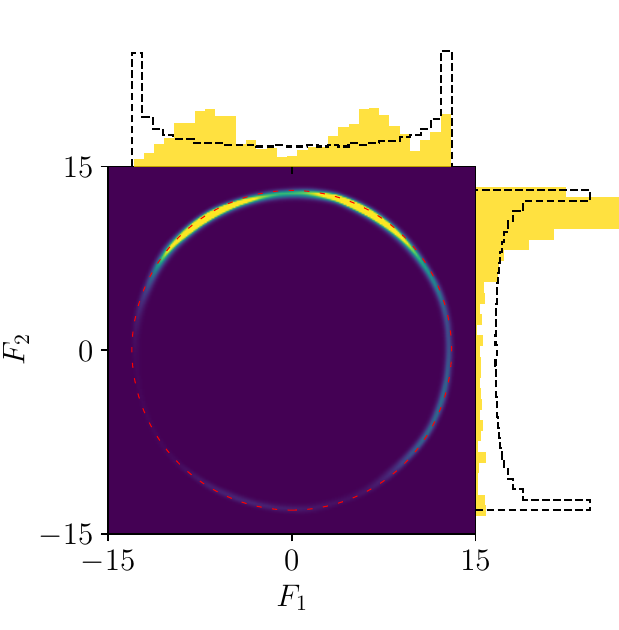}
    \caption{NRE-C}
  \end{subfigure}%
  \hfill
  \begin{subfigure}{.37\textwidth}
    \centering
    \includegraphics[trim={0.3cm 0.5cm 3cm 6.6cm},clip, height=4.5cm]{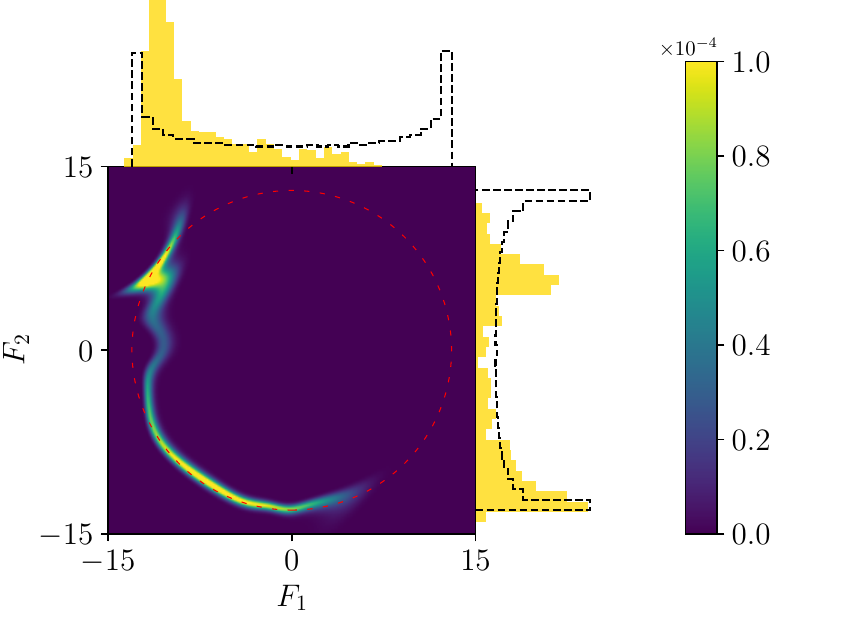}
    \caption{NPE-C}
  \end{subfigure}
\end{minipage}

\caption{\textbf{Visual comparison of the estimated joint and marginal posterior distributions for additional test samples.} Each row corresponds to one sample and shows posteriors from \eande, NRE-C \citep{miller2022contrastive}, and NPE-C \citep{greenberg2019automatic}.
}
\label{fig:posterior_11}
\end{figure}

\begin{figure}[ht!]
    \centering
    \includegraphics[width=0.5\linewidth]{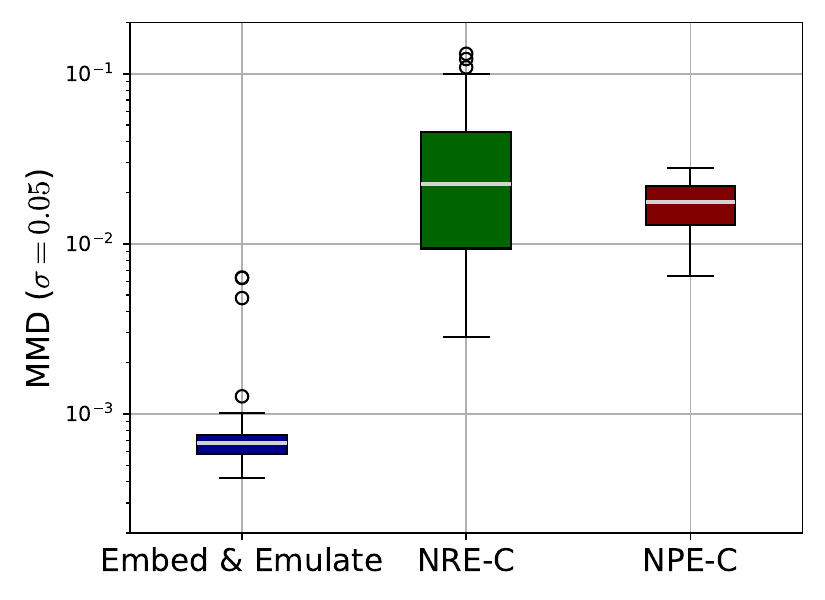}
    \caption{\textbf{Comparison of sample quality using maximum mean discrepancy over 50 testing instances (MMD).} Each box plot shows the median (25th, 75th percentiles) of the error statistics.
    We compare \eande with NRE-C \citep{ miller2022contrastive} and NPE-C \citep{greenberg2019automatic}.
    With a larger Gaussian kernel width $\sigma = 0.05$ for the MMD, the differences between methods are more pronounced.
    }
    \label{fig:mmd_05}
\end{figure}

We provide additional visualizations of the estimated posterior in Figure~\ref{fig:posterior_11}.
Furthermore, we utilize the samples obtained in Section~\ref{sec:l96-main} to calculate the Maximum Mean Discrepancy (MMD) between the posterior distributions of the learned model and the actual reference distribution. This calculation employs a Gaussian kernel with an increased width
of $\sigma = 0.05$.
As shown in Figure~\ref{fig:mmd_05}, 
with a larger kernel width, 
\eande achieves a consistently lower error with a significantly reduced variance.

\pagebreak

\bibliographystyle{unsrtnat}
\bibliography{reference}  

\end{document}